\documentclass{article}

     \usepackage[preprint,nonatbib]{neurips_2020}

\usepackage{todonotes}
\usepackage[utf8]{inputenc} 
\usepackage[T1]{fontenc}    
\usepackage{hyperref}       
\usepackage{url}            
\usepackage{booktabs}       
\usepackage{amsfonts}       
\usepackage{nicefrac}       
\usepackage{microtype}      
\usepackage{enumitem}
\usepackage{amsmath}
\usepackage{amssymb}
\usepackage{amsthm}
\usepackage{bbm}
\usepackage[ruled,vlined]{algorithm2e}
\usepackage{graphicx}
\usepackage{xcolor}
\usepackage{algcompatible,lipsum}
\newsavebox{\algleft}
\newsavebox{\algright}

\DeclareMathOperator*{\argmin}{arg\,min}

\newcommand{\E}[1]{\mathbb{E} \left[#1\right]}

\newtheorem{theorem}{Theorem}
\newtheorem{lemma}{Lemma}

\newtheorem{corollary}{Corollary}
\newtheorem{assumption}{Assumption}

\newtheorem{example}{Example}
\newcommand\blfootnote[1]{%
  \begingroup
  \renewcommand\thefootnote{}\footnote{#1}%
  \addtocounter{footnote}{-1}%
  \endgroup
}

\title{UFO-BLO: Unbiased First-Order Bilevel Optimization}

\author{%
Valerii Likhosherstov$^{*1}$, Xingyou Song$^{*2}$, Krzysztof Choromanski$^{2}$ \\
  \textbf{Jared Davis$^{3}$, Adrian Weller$^{14}$} \\
 $^1$University of Cambridge $^2$Google Brain $^3$DeepMind $^4$Alan Turing Institute\\
}

\begin{document}

\maketitle

\begin{abstract}
Bilevel optimization (BLO) is a popular approach with many applications including hyperparameter optimization, neural architecture search, adversarial robustness and model-agnostic meta-learning. However, the approach suffers from time and memory complexity proportional to the length $r$ of its inner optimization loop, which has led to  several modifications being proposed. One such modification is \textit{first-order} BLO (FO-BLO) which approximates outer-level gradients by zeroing out second derivative terms, yielding significant speed gains and
requiring only constant memory as $r$ varies. Despite FO-BLO's popularity, there is a lack of theoretical understanding of its convergence properties. We make progress by demonstrating a rich family of examples where FO-BLO-based stochastic optimization does not converge to a stationary point of the BLO objective. We address this concern by proposing a new FO-BLO-based unbiased estimate of outer-level gradients, enabling us to theoretically guarantee this convergence, with no harm to memory and expected time complexity. 
Our findings are supported by experimental results on Omniglot and Mini-ImageNet, popular few-shot meta-learning benchmarks.\footnote{A revised version of this manuscript (\href{https://arxiv.org/abs/2106.02487}{https://arxiv.org/abs/2106.02487}) has been accepted to ICML 2021.}
\end{abstract}
\blfootnote{$^\ast$Equal contribution.}
\footnotetext[14]{\texttt{\{vl304,aw665\}@cam.ac.uk}}
\footnotetext[23]{\texttt{\{xingyousong,kchoro,jaredquincy\}@google.com}}

\section{Introduction}

Bilevel optimization (BLO) is a popular technique of defining an outer-level objective function through the result of an inner optimizer's loop. BLO finds many applications in various subtopics of deep learning, including hyperparameter optimization \cite{hyperopt,hyperoptimpl,truncated,revopt}, neural architecture search \cite{darts}, adversarial robustness \cite{advrobbilevel,advrobgraphs} and meta-learning \cite{maml,truncated}.

In a standard setup, the outer-level optimization is conducted via stochastic gradient descent (SGD) where gradients are obtained by automatic differentiation \cite{autograd} through the $r$ steps of the inner-level gradient descent (GD). This implies the need to store all $r$ intermediate inner-optimization states in memory in order to review them during back-propagation. Thus, longer inner-GD lengths $r$ can be prohibitively expensive. To address this, several approximate versions of BLO were proposed. Among them is truncated back-propagation \cite{truncated} which only stores a fixed amount of the last inner optimization steps. Implicit differentiation \cite{impl1,impl2,imaml,hyperoptimpl} takes advantage of the Implicit Function Theorem, application of which however comes with a list of restrictions to be fulfilled (see Section \ref{sec:comp}). Designing optimizer's iteration as a reversible dynamics \cite{revopt} allows underlying-constant reduction in $O(r)$ memory complexity since only the bits lost in fixed-precision arithmetic operations are saved in memory. Forward differentiation \cite{fordiff} can be employed when the outer-level gradients for only a small number of parameters are collected. Checkpointing \cite{ckpt} is a generic solution for memory reduction by a factor of $\sqrt{r}$.

Arguably the most practical and simplest method is a first-order BLO approximation (FO-BLO), where second-order computations are omitted, and hence rollout states are not stored in memory but only attended once. FO-BLO is widely used in the context of neural architecture search \cite{darts}, adversarial robustness \cite{advrobbilevel,advrobgraphs} and meta-learning \cite{maml}. FO-BLO can be viewed as a limit form of truncated back-propagation where no states are cached in memory. We argue, however, that the cost of FO-BLO computational simplicity is high: we show that FO-BLO \textit{can fail to converge to a stationary point} of the BLO objective. Consequently, it would be highly desirable to modify FO-BLO so that, under the same time and memory complexity, it would possess better convergence properties.

We propose a solution: a FO-BLO-based algorithm which, with no harm to time and memory complexity, benefits from unbiased gradient estimation. To achieve this, we first propose a method to compute precise BLO gradients using only constant memory at the cost of quadratic time $O(r^2)$. We then combine FO-BLO with this slow, but memory-efficient exact BLO, used as a stochastic gradient correction computed at random with probability $\propto r^{-1}$ per outer-loop iteration, to yield a convergent training algorithm. We call our algorithm unbiased first-order BLO (UFO-BLO) and highlight the following benefits:
\begin{itemize}[noitemsep,nolistsep] 
    \item UFO-BLO has the same $O(1)$ memory and (expected) $O(r)$ time complexity as FO-BLO (Theorem \ref{th:fomamltm}, Corollary \ref{cor:ufomamltm}); strictly better than the complexity of explicit BLO ($O(r)$ time and memory, Theorem \ref{th:mamltm}).
    \item Under mild assumptions, UFO-BLO with stochastic outer-loop optimization is guaranteed to converge to a stationary point of the BLO objective (Theorem \ref{th:conv}).
    \item Confirming the need for UFO-BLO, we prove that, under the same mild assumptions, there is a rich family of BLO problem formulations where stochastic FO-BLO optimization ends up arbitrarily far away from a stationary point (Theorem \ref{th:counter}).
\end{itemize}
In addition to theoretical contributions, we demonstrate the utility of UFO-BLO on the popular Omniglot \cite{omniglot} and Mini-ImageNet \cite{imagenet} benchmarks, standard settings for few-shot learning.

\begin{figure}[b]
  \centering
  \includegraphics[width=\textwidth]{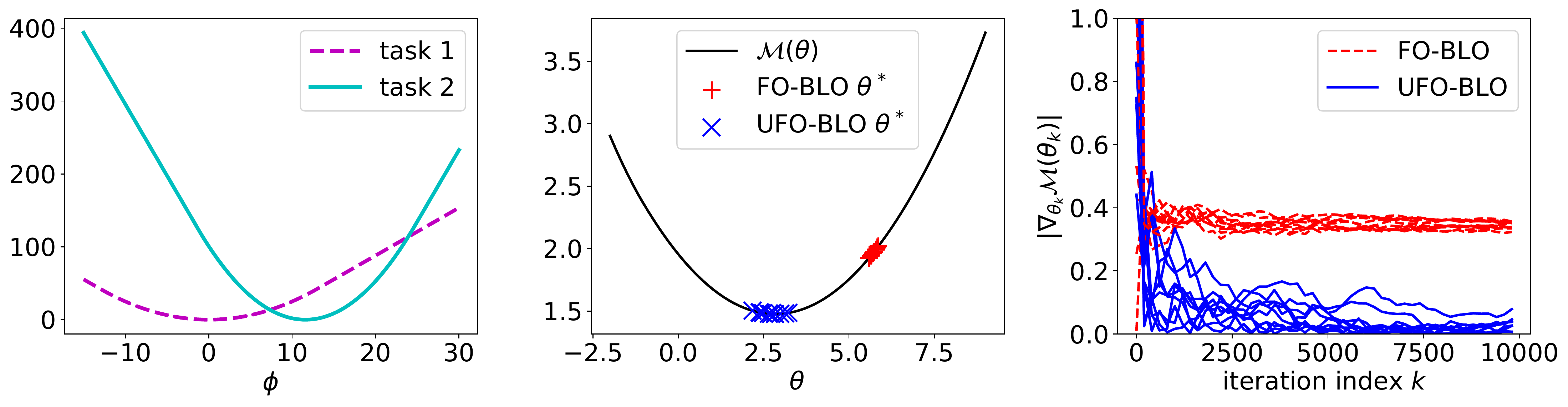}
  \caption{Example of a BLO problem where FO-BLO fails to converge to a stationary point of the optimization objective $\mathcal{M} (\theta)$. \textbf{(Left)} Two tasks $\mathcal{T}_1$, $\mathcal{T}_2$ are sampled equiprobably from $p(\mathcal{T})$. The goal of $i$'s task is to optimize $f_i (\phi)$ with respect to $\phi \in \mathbb{R}$. \textbf{(Middle)} Resulting BLO objective $\mathcal{M} (\theta)$ where $\theta$ is a starting parameter for $(r = 10)$-step inner optimization loop. Markers ``+'', ``x'' indicate results of $\mathcal{M} (\theta)$ mini-batch optimization (meta-batch size of 1) from a random starting parameter $\theta_0$ using FO-BLO and UFO-BLO ($q = 0.1$) respectively. \textbf{(Right)} Convergence of $| \nabla_{\theta_k} \mathcal{M} (\theta_k) |$ where $k$ is an index of outer-loop iteration. $\nabla_{\theta_k} \mathcal{M} (\theta_k)$ is approaching zero as UFO-BLO progresses which is not true for FO-BLO. More details about the experimental setup can be found in Appendix \ref{sec:synth} of the Supplement.}
  \label{fig:synth}
\end{figure}

\begin{figure}[ht]
\begin{minipage}{0.49\textwidth}
\begin{algorithm}[H]
\SetAlgoLined
\SetKwInOut{Input}{Input}
\Input{$\theta_0 \in \mathbb{R}^p$, $\tau, v \in \mathbb{N}$.}

\For{$k \gets 1$ \KwTo $\tau$}{
    Set $b = \mathbf{0}_u$\;
    \For{$w \gets 1$ \KwTo $v$}{
        Draw $\mathcal{T} = ( \mathcal{D}^{tr}, \mathcal{D}^{test} ) \sim p (\mathcal{T})$\;
        Set $b := b + \mathcal{G} (\theta_{k - 1}, \mathcal{T})$\;
    }
    Set $\theta_k := \theta_{k - 1} - \frac{\gamma_k}{v} b$\;
}
\caption{Outer mini-batch GD.}
\label{alg:sgd}
\end{algorithm}
\vspace{11.3pt}

\begin{algorithm}[H]
\SetAlgoLined
\SetKwInOut{Input}{Input}
\Input{$\theta \in \mathbb{R}^p$, $\alpha > 0$, $r \in \mathbb{N}$, $\mathcal{T}$.}
\KwResult{$\mathcal{G}_{FO} (\theta, \mathcal{T})$.}

Set $\phi := \theta$\;

\For{$j \gets 1$ \KwTo $r$}{
    Set $\phi := \phi - \alpha \nabla_\phi \mathcal{L}^{in} (\phi, \mathcal{T})$\;
}

\Return $\nabla_\phi \mathcal{L}^{out} (\phi, \mathcal{T})$\;

\caption{Inner GD (FO-BLO).}
\label{alg:fomaml}
\end{algorithm}
\vspace{11.3pt}

\begin{algorithm}[H]
\SetAlgoLined
\SetKwInOut{Input}{Input}
\Input{$\theta \in \mathbb{R}^p$, $\alpha > 0$, $r \in \mathbb{N}$, $\mathcal{T}$, $q \in (0, 1]$.}
\KwResult{$\mathcal{G}_{UFO} (\theta, \mathcal{T})$.}

Let $b := $ result of Algorithm \ref{alg:fomaml}\;

Draw $\xi \sim \mathrm{Bernoulli} (q)$\;

\eIf{$\xi = 1$}{
    Let $c := $ result of Algorithm \ref{alg:memmaml}\;
    \Return $b + \frac{1}{q} (c - b)$\;
}{
    \Return $b$\;
}

\caption{Inner GD (UFO-BLO).}
\label{alg:unbiased}
\end{algorithm}

\end{minipage}
\hfill
\begin{minipage}{0.49\textwidth}
\begin{algorithm}[H]
\SetAlgoLined
\SetKwInOut{Input}{Input}
\Input{$\theta \in \mathbb{R}^p$, $\alpha > 0$, $r \in \mathbb{N}$, $\mathcal{T}$, $\mathcal{T}$.}
\KwResult{$\mathcal{G} (\theta, \mathcal{T}) = \nabla_\theta \mathcal{L}^{out} (U (\theta, \mathcal{T}), \mathcal{T})$.}

New array $\textrm{Arr}[1..r]$\;

Set $\phi := \theta$\;

\For{$j \gets 1$ \KwTo $r$}{
    Set $\textrm{Arr} [j] := \phi$\;
    Set $\phi := \phi - \alpha \nabla_\phi \mathcal{L}^{in} (\phi, \mathcal{T})$\;
}

Set $b := \nabla_\phi \mathcal{L}^{out} (\phi, \mathcal{T})$\;

\For{$j \gets r$ \KwTo $1$}{
    Set $\phi := \textrm{Arr} [j]$\;
    Set $b := b - \alpha ( \nabla^2_\phi \mathcal{L}^{in} (\phi, \mathcal{T})^\top b )$\;
}
\Return $b$\;
\caption{Inner GD (BLO).}
\label{alg:maml}
\end{algorithm}

\begin{algorithm}[H]
\SetAlgoLined
\SetKwInOut{Input}{Input}
\Input{$\theta \in \mathbb{R}^p$, $\alpha > 0$, $r \in \mathbb{N}$, $\mathcal{T}$, $\mathcal{T}$.}
\KwResult{$\mathcal{G} (\theta, \mathcal{T}) = \nabla_\theta \mathcal{L}^{out} (U (\theta, \mathcal{T}), \mathcal{T})$.}

Set $\phi := \theta$\;

\For{$j \gets 1$ \KwTo $r$}{
    Set $\phi := \phi - \alpha \nabla_\phi \mathcal{L}^{in} (\phi, \mathcal{T})$\;
}

Set $b := \nabla_\phi \mathcal{L}^{out} (\phi, \mathcal{T})$ \;

\For{$j_1 \gets r$ \KwTo $1$}{

    Set $\phi := \theta$\;

    \For{$j_2 \gets 1$ \KwTo $j_1 - 1$}{
        Set $\phi := \phi - \alpha \nabla_\phi \mathcal{L}^{in} (\phi, \mathcal{T})$\;
    }

    Set $b := b - \alpha ( \nabla^2_\phi \mathcal{L}^{in} (\phi, \mathcal{T})^\top b )$\;
}
\Return $b$\;
\caption{Memory-efficient inner GD (BLO).}
\label{alg:memmaml}
\end{algorithm}

\end{minipage}
\end{figure}
\section{Preliminaries}

We commence by formulating the bilevel optimization (BLO) problem and the exact algorithm. We then discuss first-order BLO, and few-shot learning as a prominent example of the BLO problem.

\subsection{Bilevel optimization: problem statement}

We formulate the bilevel optimization (BLO) problem in a form which is compatible with prominent large-scale deep-learning applications. Namely, let $\Omega_\mathcal{T}$ be a non-empty set of \textit{tasks} $\mathcal{T} \in \Omega_\mathcal{T}$ and let $p(\mathcal{T})$ be a probabilistic distribution defined on $\Omega_\mathcal{T}$. In practice the procedure of sampling a task $\mathcal{T} \sim p(\mathcal{T})$ is usually resource-cheap and consists of retrieving a task from disk using either a random index (for a dataset of a fixed size), a stream, or a simulator. Define two functions as the inner and outer loss respectively: $\mathcal{L}^{in}, \mathcal{L}^{out}: \mathbb{R}^p \times \Omega_\mathcal{T} \to \mathbb{R}$. In addition, define $U: \mathbb{R}^p \times \Omega_\mathcal{T} \to \mathbb{R}^p$ so that $U (\theta, \mathcal{T})$ is a result of $r$-step inner (task-level) gradient descent (GD) minimizing inner loss $\mathcal{L}^{in}$:
\begin{equation}
    U (\theta, \mathcal{T}) = \phi_r,  \quad \forall j \in \{ 1, \dots, r \}: \phi_j = \phi_{j - 1} - \alpha \nabla_{\phi_{j - 1}} \mathcal{L}^{in} (\phi_{j - 1}, \mathcal{T}), \quad \phi_0 = \theta , \label{eq:maml}
\end{equation}
where $\alpha > 0$ is a GD step size and $\nabla$ is denoted as the gradient. Then the BLO problem is defined as finding an initialization $\theta \in \mathbb{R}^p$ to $U (\theta, \cdot)$ which minimizes the expected outer loss $\mathcal{L}^{out}$:
\begin{equation}
    \min_\theta  \mathcal{M} (\theta) \> \>  \text{where} \> \> \mathcal{M}(\theta) = \mathbb{E}_{p (\mathcal{T})} \left[\mathcal{L}^{out} (U (\theta, \mathcal{T}), \mathcal{T}) \right] \label{eq:opt}
\end{equation}
The formulations (\ref{eq:maml}-\ref{eq:opt}) unify various application scenarios. By splitting vector $\phi = [ \phi_1^\top \, \phi_2^\top ]^\top$ where $\phi_1 \in \mathbb{R}^u$, $u < p$, and assuming that $\mathcal{L}^{in} (\phi, \cdot), \mathcal{L}^{out} (\phi, \cdot)$ only depend on $\phi_1$ and $\phi_2$ respectively, we can think of $\phi_1$ as model parameters and (\ref{eq:maml}) as a training loop, while $\phi_2$ are hyperparameters optimized in the outer loop -- a scenario matching the hyperparameter optimization paradigm \cite{hyperopt,hyperoptimpl,truncated,revopt}. Alternatively, $\phi_2$ may act as parameters encoding a neural-network's topology which corresponds to differentiable neural architecture search \cite{darts}. Here $\mathcal{T}$ encodes a minibatch drawn from a dataset. The adversarial robustness problem fits into (\ref{eq:maml}-\ref{eq:opt}) by assuming that $\phi_1$ are learned model parameters and $\phi_2$ is an input perturbation optimized through (\ref{eq:maml}) to decrease the model's performance. See Section \ref{sec:fewshot} for a formalization of few-shot meta-learning expressed as BLO (\ref{eq:maml}-\ref{eq:opt}).

We consider mini-batch gradient descent \cite{bottou} as a solver for (\ref{eq:opt}) -- see Algorithm \ref{alg:sgd}. Let $\mathbf{0}_p$ denote a vector of $p$ zeros. Furthermore $\mathcal{G} (\theta, \mathcal{T})$ is either an exact gradient $\nabla_\theta \mathcal{L}^{out} (U (\theta, \mathcal{T}), \mathcal{T})$ or its approximation and $\{ \gamma_k > 0 \}_{k = 1}^\infty$ denotes a sequence of outer-loop step sizes which satisfies
\begin{equation} \label{eq:step}
    \sum_{k = 1}^\infty \gamma_k = \infty, \quad \lim_{k \to \infty} \gamma_k = 0.
\end{equation}

Assuming that the procedure for sampling tasks $\mathcal{T} \sim p(\mathcal{T})$ takes negligible resources, the time and memory requirements of Algorithm \ref{alg:sgd} are dominated by the time spent and space allocated for computing $\mathcal{G} (\theta, \mathcal{T})$. Therefore, in our subsequent derivations, we analyse the computational complexity of finding the gradient estimate $\mathcal{G} (\theta, \mathcal{T})$.

\subsection{Exact BLO gradients} \label{sec:maml}

Outer mini-batch GD requires the computation or approximation of a gradient $\nabla_\theta \mathcal{L}^{out} (U (\theta, \mathcal{T}), \mathcal{T}) = \nabla_{\phi_0} \mathcal{L}^{out} (\phi_r, \mathcal{T})$. We apply the chain rule to the inner GD (\ref{eq:maml}) and deduce that for $j \in \{ 1, \dots, r \}$:
\begin{equation}
    \nabla_{\phi_{j - 1}} \mathcal{L}^{out} (\phi_r, \mathcal{T}) \!=\! \frac{\partial \phi_j}{\partial \phi_{j - 1}}^\top \!\nabla_{\phi_j} \mathcal{L}^{out} (\phi_r, \mathcal{T}) = ( I \!- \alpha \nabla^2_{\phi_{j - 1}} \mathcal{L}^{in} (\phi_{j - 1}, \mathcal{T})^\top ) \nabla_{\phi_j} \mathcal{L}^{out} (\phi_r, \mathcal{T}) , \label{eq:mamlgrad}
\end{equation}
where $\frac{\partial \phi_j}{\partial \phi_{j - 1}}$ denotes a $p \times p$ Jacobian matrix of $\phi_j$ with respect to $\phi_{j - 1}$, $\nabla^2$ is a $p \times p$ Hessian matrix, and $I$ is the identity matrix of appropriate size. Based on (\ref{eq:maml},\ref{eq:mamlgrad}), Algorithm \ref{alg:maml} illustrates the inner computation for exact BLO, which, together with Algorithm \ref{alg:sgd}, outlines the training procedure.

In standard deep learning applications, $\mathcal{L}^{in} (\phi, \mathcal{T}), \mathcal{L}^{out} (\phi, \mathcal{T})$ are computed by explicitly evaluating a computation graph. Therefore, automatic differentiation (implemented in Tensorflow \cite{tensorflow} and PyTorch \cite{pytorch}) allows computation of $\nabla_\phi \mathcal{L}^{in} (\phi, \mathcal{T}), \nabla_\phi \mathcal{L}^{out} (\phi, \mathcal{T})$ in time and memory only a constant factor bigger than needed to evaluate $\mathcal{L}^{in} (\phi, \mathcal{T}), \mathcal{L}^{out} (\phi, \mathcal{T})$ respectively (the \textit{Cheap Gradient Principle} \cite{autograd}). The same is true for Hessian-vector products $\nabla^2_\phi \mathcal{L}^{in} (\phi, \mathcal{T})^\top b$ which, for any $b \in \mathbb{R}^p$, can be computed using the reverse accumulation technique \cite{hessian} without explicitly constructing the Hessian matrix $\nabla^2_\phi \mathcal{L}^{in} (\phi, \mathcal{T})$. This technique consists of evaluation and automatic differentiation of a functional $H: \mathbb{R}^p \to \mathbb{R}$, $H(\phi) = \nabla_\phi \mathcal{L}^{in} (\phi, \mathcal{T})^\top b$. The result of differentiation is precisely $\nabla_\phi H(\phi) = \nabla^2_\phi \mathcal{L}^{in} (\phi, \mathcal{T})^\top b$. For simplicity we omit formal definitions and derivations which can be found in the dedicated literature \cite{hessian,autograd} and hereafter make the following assumption:
\begin{assumption} \label{as:compunit}
    Let $C_\mathrm{T}, C_\mathrm{M}$ denote an upper bound on the time and memory respectively required to evaluate $\mathcal{L}^{in} (\phi, \mathcal{T})$ and $\mathcal{L}^{out} (\phi, \mathcal{T})$ for any $\phi \in \mathbb{R}^p, \mathcal{T} \in \Omega_\mathcal{T}$. Then the time and memory required to compute $\nabla_\phi \mathcal{L}^{in} (\phi, \mathcal{T})$, $\nabla_\phi \mathcal{L}^{out} (\phi, \mathcal{T})$, $\nabla^2_\phi \mathcal{L}^{in} (\phi, \mathcal{D})^\top b$ for any $\phi, b \in \mathbb{R}^p, \mathcal{T} \in \Omega_\mathcal{T}$ are upper bounded by $C_\mathrm{T}, C_\mathrm{M}$ respectively, multiplied by a universal constant.
\end{assumption}
The following theorem follows naturally from analysis of Algorithm \ref{alg:maml}:
\begin{theorem} \label{th:mamltm}
    Under Assumption \ref{as:compunit}, the time and memory complexities of Algorithm \ref{alg:maml} are $O(r \cdot p + r \cdot C_\mathrm{T})$ and $O(r \cdot p + C_\mathrm{M})$ respectively.
\end{theorem}

\subsection{First-order approximation}

$O(r \cdot p + C_M)$ memory complexity is a limitation which can significantly complicate application of BLO in real-life scenarios when both the number of parameters $p$ and the number of gradient descent iterations $r$ are large. A number of improvements have been proposed in the literature to overcome this issue \cite{truncated,revnet,hyperoptimpl,imaml}. A simple method which sometimes performs well in practice is \textit{first-order} BLO (FO-BLO) \cite{maml,darts,advrobbilevel,advrobgraphs}; this proposes to approximate $\nabla_\theta \mathcal{L}^{out} (U (\theta, \mathcal{T}), \mathcal{T})$ with $\nabla_{\phi_r} \mathcal{L}^{out} (\phi_r, \mathcal{T})$ corresponding to zeroing out Hessians in Equation (\ref{eq:mamlgrad}). Since only the last-step gradient is important, there is no need to store states $\phi_1, \dots, \phi_r$ -- see Algorithm \ref{alg:fomaml}. The time and memory complexities of FO-BLO are formalized as follows:
\begin{theorem} \label{th:fomamltm}
    Under Assumption \ref{as:compunit}, the time and memory complexities of Algorithm \ref{alg:fomaml} are $O(r \cdot p + r \cdot C_\mathrm{T})$ and $O(p + C_\mathrm{M})$ respectively.
\end{theorem}

Although FO-BLO enjoys a better memory complexity and can sometimes perform well in practice, it can fail to converge to a stationary point of the BLO objective (\ref{eq:opt}) (see Section \ref{sec:theory}).

\subsection{Few-shot meta-learning: example of BLO problem} \label{sec:fewshot}

\textit{Few-shot meta-learning} is a celebrated example \cite{maml} of a BLO problem. It addresses adaptation to a new task when supplied with a small amount of training data. Define $\mathcal{X} = \mathbb{R}^n, \mathcal{Y} = \mathbb{R}^m$ as the observation and prediction domains respectively. Each task $\mathcal{T}$ is a pair defined as:
\begin{gather*}
    \mathcal{T} = ( \mathcal{D}^{tr}_\mathcal{T}, \mathcal{D}^{test}_\mathcal{T} ), \mathcal{D}^{tr}_\mathcal{T} = ( ( X^{tr}_i, Y^{tr}_i ) )_{i = 1}^s \in (\mathcal{X} \times \mathcal{Y})^s, \mathcal{D}^{test}_\mathcal{T} = ( ( X^{test}_i, Y^{test}_i ) )_{i = 1}^t \in (\mathcal{X} \times \mathcal{Y})^t ,
\end{gather*}
where $\mathcal{D}^{tr}_\mathcal{T}$ is a training set of a typically small size $s$ (number of \textit{shots}), and $\mathcal{D}^{test}_\mathcal{T}$ is a test set of size $t$. Therefore, $\Omega_\mathcal{T} = ((\mathcal{X} \times \mathcal{Y})^s) \times ( (\mathcal{X} \times \mathcal{Y})^t)$.

We consider an $m$-class few-shot classification problem in particular. That is, $p(\mathcal{T})$ is nonzero only when the corresponding $Y_i^{tr}$ are class one-hot encodings with a single nonzero entry of $1$ encoding the class. Let $g: \mathbb{R}^p \times \mathcal{X} \to \mathcal{Y}$, be an \textit{estimator} (e.g. a feed-forward or a convolutional neural network) with parameters $\phi$ and input $X$. $g (\phi, \cdot)$ outputs label logits which are fed into \textit{categorical cross-entropy loss} $l_\mathrm{CCE} (Z, Y) = \log ( \sum_{c = 1}^m \exp (Z_c) ) - Z^\top Y$.

Define $\mathcal{L}_\mathrm{CCE} (\phi, \mathcal{D}) = \frac{1}{| \mathcal{D} |} \sum_{(X, Y) \in \mathcal{D}} l_\mathrm{CCE} (g(\phi, X), Y)$. \textit{Model-agnostic meta-learning} (MAML) \cite{maml} states the problem of few-shot classification as BLO (\ref{eq:maml}-\ref{eq:opt}) where $\mathcal{L}^{in} (\phi, \mathcal{T}) = \mathcal{L}_\mathrm{CCE} (\phi, \mathcal{D}_\mathcal{T}^{tr})$ and $\mathcal{L}^{out} (\phi, \mathcal{T}) = \mathcal{L}_\mathrm{CCE} (\phi, \mathcal{D}_\mathcal{T}^{test})$. This way, inner GD corresponds to fitting $g (\phi, \cdot)$ to a training set $\mathcal{D}_\mathcal{T}^{tr}$ of a small size, while the outer mini-batch GD is searching for an initialization $\theta = \phi_0$ maximizing generalization on the unseen data $\mathcal{D}_\mathcal{T}^{test}$.

\begin{example}
$C_\mathrm{T} = O( \max (s, t) \cdot p ), C_\mathrm{M} = O(p)$ for the definition of $\Omega_\mathcal{T}$, $\mathcal{L}^{in} (\phi, \mathcal{T}), \mathcal{L}^{out} (\phi, \mathcal{T})$ as above. See Appendix \ref{sec:ffn} for further discussion.
\end{example}
\section{Unbiased first-order bilevel optimization (UFO-BLO)} \label{sec:prop}

An alternative way to compute $\nabla_\theta \mathcal{L}^{out} (U (\theta, \mathcal{T}), \mathcal{T})$ without storing the array of inner-GD intermediate states $\phi_1, \dots, \phi_r$ is illustrated in Algorithm \ref{alg:memmaml}, where each $\phi_j$, $1 \leq j \leq r$, is recomputed when needed, using a nested loop inside a backward pass. Hence, memory efficiency comes at the cost of quadratic running time complexity. Algorithm \ref{alg:memmaml} alone, however, does not give a practical way to solve the optimization problem (\ref{eq:opt}). Instead, we show how to combine Algorithm \ref{alg:memmaml} with FO-BLO into a randomized scheme with tractable complexity bounds and convergence guarantees.

Let $\xi \in \{ 0, 1 \}$ be a Bernoulli random variable with $\mathbb{P} (\xi = 1) = q$ (denote as $\xi \sim \mathrm{Bernoulli} (q)$) where $q \in (0, 1]$. Recall that $\mathcal{G}_{FO}$ is the first-order gradient from Algorithm \ref{alg:fomaml}. We consider the following stochastic approximation to $\nabla_\theta \mathcal{L}^{out} (U (\theta, \mathcal{T}), \mathcal{T})$:
\begin{equation}
    \mathcal{G}_{UFO} (\theta, \mathcal{T}) = \mathcal{G}_{FO} (\theta, \mathcal{T}) + (\xi / q) (\nabla_\theta \mathcal{L}^{out} (U (\theta, \mathcal{T}), \mathcal{T}) - \mathcal{G}_{FO} (\theta, \mathcal{T})) . \label{eq:unb}
\end{equation}
In fact, (\ref{eq:unb}) is an unbiased estimate of $\nabla_\theta \mathcal{L} (U (\theta, \mathcal{T}), \mathcal{T})$. Indeed, since $\mathbb{E}_\xi \left[\xi\right] = q$:
\begin{equation*}
    \mathbb{E}_\xi \left[\mathcal{G}_{UFO} (\theta, \mathcal{T}) \right] = (1 - q / q) \mathcal{G}_{FO} (\theta, \mathcal{T}) + (q / q) \nabla_\theta \mathcal{L}^{out} (U (\theta, \mathcal{T}), \mathcal{T}) = \nabla_\theta \mathcal{L}^{out} (U (\theta, \mathcal{T}), \mathcal{T}) .
\end{equation*}
For this reason we call the estimate (\ref{eq:unb}) \textit{unbiased first-order} BLO (UFO-BLO). Algorithm \ref{alg:unbiased} illustrates randomized computation of UFO-BLO. It combines FO-BLO (Algorithm \ref{alg:fomaml}) and memory-efficient BLO (Algorithm \ref{alg:memmaml}) and is therefore also memory-efficient. In addition, for certain values of $q$, Algorithm \ref{alg:unbiased} becomes running time efficient:
\begin{theorem} \label{th:ufomamlrm}
    Under Assumption \ref{as:compunit}, the expected running time of Algorithm \ref{alg:unbiased} is $O(r \cdot p + r \cdot C_\mathrm{T} + q \cdot r^2 \cdot p + q \cdot r^2 \cdot C_\mathrm{T})$ while memory complexity is $O(p + C_\mathrm{M})$.
\end{theorem}
\begin{proof}
    Memory complexity follows naturally from the algorithm's definition. The running time of the algorithm satisfies a randomized upper bound $O(r \cdot p + r \cdot C_\mathrm{T} + \xi \cdot r^2 \cdot p + \xi \cdot r^2 \cdot C_\mathrm{T})$. The theorem is obtained by taking expectation of the running time and its upper bound.
\end{proof}

\begin{corollary} \label{cor:ufomamltm}
    Let $q = \frac{C}{r}$ where $0 < C \leq r$ is a universal constant. Then under Assumption \ref{as:compunit}, the expected running time of Algorithm \ref{alg:unbiased} is $O(r \cdot p + r \cdot C_\mathrm{T})$ while memory complexity is $O(p + C_\mathrm{M})$.
\end{corollary}

By the law of large numbers \cite{prob}, the time complexity of UFO-BLO approaches its expected value when $\tau$, the number of outer iterations, is large ($\tau \gg r$), which is typical for large-scale problems.
\section{Convergence results} \label{sec:theory}

In this section we first provide convergence guarantees for UFO-BLO under a set of broad, nonconvex assumptions (Theorem \ref{th:conv}). We analyse UFO-BLO as an algorithm which finds a stationary point of the BLO objective (\ref{eq:maml}-\ref{eq:opt}), i.e. a point $\theta^* \in \mathbb{R}^p$ such that $\nabla_{\theta^*} \mathcal{M} (\theta^*) = \mathbf{0}_p$. Motivated by stationary point $\theta^*$ search, we prove a standard result for stochastic optimization of nonconvex functions \cite[Section 4.3]{bottou}, that is
\begin{equation} \label{eq:liminf}
    \liminf_{k \to \infty} \E{ \| \nabla_{\theta_k} \mathcal{M} (\theta_k) \|_2^2} = 0 ,
\end{equation}
where $\theta_k$ are iterates of mini-batch GD with UFO-BLO gradient estimation. Intuitively, equation (\ref{eq:liminf}) implies that there exist iterates of UFO-BLO which approach some stationary point $\theta^*$ up to any level of proximity. Our second contribution is a rigorous proof that equation \textit{(\ref{eq:liminf}) does not hold for FO-BLO under the same assumptions} (Theorem \ref{th:counter}). More specifically, we show that for any $D > 0$, there exists an optimization problem of type (\ref{eq:opt}) such that $\liminf_{k \to \infty} \E{\| \nabla_{\theta_k} \mathcal{M} (\theta_k) \|_2^2} > D$ where $\{ \theta_k \}$ are iterates of FO-BLO. The intuition behind this result is that FO-BLO cannot find a solution with gradient norm lower than $D$. We first formulate Assumptions \ref{as:liphes}, \ref{as:mreg} which we use for proofs.

\begin{assumption}[Uniformly bounded, uniformly Lipschitz-continuous gradients and Hessians] \label{as:liphes}
    For any $\mathcal{T} \in \Omega_\mathcal{T}$, $\mathcal{L}^{in} (\phi, \mathcal{T}), \mathcal{L}^{out} (\phi, \mathcal{T})$ are twice differentiable as functions of $\phi$. There exist constants $L_1, L_2, L_3 > 0$ such that for any $\mathcal{T} \in \Omega_\mathcal{T}, \phi, \psi \in \mathbb{R}^p$, it holds that
    \begin{equation*}
        \| \nabla_\phi \mathcal{L}^{out} (\phi, \mathcal{T}) \|_2 \leq L_1, \quad \| \nabla^2_\phi \mathcal{L}^{in} (\phi, \mathcal{T}) -  \nabla^2_\psi \mathcal{L}^{in} (\psi, \mathcal{T}) \|_2 \leq L_3 \| \phi - \psi \|_2
    \end{equation*}
    and for $\square \!\in\! \{ in, out \}$, $\| \nabla^2_\phi \mathcal{L}^\square (\phi, \mathcal{T}) \|_2 \leq L_2$.
\end{assumption}
Observe that the following assumption is satisfied in particular when $p (\mathcal{T})$ is defined on a finite set of tasks $\mathcal{T}$ (e.g. when the meta-dataset is finite) and $\mathcal{L}$ is lower-bounded.
\begin{assumption}[Regularity of $\mathcal{M}$] \label{as:mreg}
    For each $\theta \in \mathbb{R}^p$, the terms $\mathcal{M}(\theta)$, $\nabla_\theta \mathcal{M} (\theta)$, $\mathbb{E}_{p (\mathcal{T})} \left[\nabla_\theta \mathcal{L}^{out} (U (\theta, \mathcal{T}), \mathcal{T}) \right]$ are well-defined and $\nabla_\theta \mathcal{M} (\theta) = \mathbb{E}_{p (\mathcal{T})} \left[\nabla_\theta \mathcal{L}^{out} (U (\theta, \mathcal{T}), \mathcal{T})\right]$. Let $\mathcal{M}^* = \inf_{\theta \in \mathbb{R}^p} \mathcal{M} (\theta)$, then $\mathcal{M}^* > - \infty$.
\end{assumption}

Below we formulate theoretical results (Theorems \ref{th:conv}, \ref{th:counter}) which are proved in Appendix \ref{sec:proofs}. Note that as a special case of Theorem \ref{th:conv}, we obtain a convergence proof for BLO with exact gradients (Algorithm \ref{alg:maml}). Indeed, one may simply set $q = 1$ in the statement of the theorem.

\begin{theorem}[Convergence of UFO-BLO] \label{th:conv}
Let $p, r, v \in \mathbb{N}, \alpha > 0, q \in (0, 1], \theta_0 \in \mathbb{R}^p$, $p(\mathcal{T})$ be a distribution on a nonempty set $\Omega_\mathcal{T}$, $\{ \gamma_k > 0 \}_{k = 1}^\infty$ be any sequence, $\mathcal{L}^{in}, \mathcal{L}^{out} : \mathbb{R}^p \times \Omega_\mathcal{T} \to \mathbb{R}$ be functions satisfying Assumption \ref{as:liphes}, and let $U: \mathbb{R}^p \times \Omega_\mathcal{T} \to \mathbb{R}^p$ be defined according to (\ref{eq:maml}), $\mathcal{M} : \mathbb{R}^p \to \mathbb{R}$ be defined by (\ref{eq:opt}) and satisfy Assumption \ref{as:mreg}. Define $\mathcal{G}: \mathbb{R}^p \times \Omega_\mathcal{T} \times \{ 0, 1 \} \to \mathbb{R}^p$ as
\begin{equation*}
    \mathcal{G} (\theta, \mathcal{T}, x) = \nabla_{\phi_r} \mathcal{L}^{out} (\phi_r, \mathcal{T}) + (x / q) (\nabla_\theta \mathcal{L}^{out} (\phi_r, \mathcal{T}) - \nabla_{\phi_r} \mathcal{L}^{out} (\phi_r, \mathcal{T})) ,
\end{equation*}
where $\phi_r = U(\theta, \mathcal{T})$. Let $\{ \mathcal{T}_{k,w} \}, \{ \xi_{k,w} \}, w \in \{ 1, \dots, v \}, k \in \mathbb{N}$ be sets of i.i.d. samples from $p(\mathcal{T})$ and $\mathrm{Bernoulli} (q)$ respectively, such that $\sigma$-algebras populated by $\{ \mathcal{T}_{k,w} \}_{\forall k, w}, \{ \xi_{k,w} \}_{\forall k, w}$ are independent.
Let $\{ \theta_k \in \mathbb{R}^p \}_{k = 0}^\infty$ be a sequence where for all $k \in \mathbb{N}$ $\theta_k = \theta_{k - 1} - \frac{\gamma_k}{v} \sum_{w = 1}^v \mathcal{G} (\theta_{k - 1}, \mathcal{T}_{k,w}, \xi_{k,w})$. Then it holds that
\begin{enumerate}
    \item If $\{ \gamma_k \}_{k = 1}^\infty$ satisfies (\ref{eq:step}) and $\sum_{k = 1}^\infty \gamma_k^2 < \infty$, then $\liminf_{k \to \infty} \E{\| \nabla_{\theta_k} \mathcal{M} (\theta_k) \|_2^2} = 0$;
    \item If $\forall k \in \mathbb{N}: \gamma_k = k^{-0.5}$, then $\min_{0 \leq u < k} \E{ \| \nabla_{\theta_u} \mathcal{M} (\theta_u) \|_2^2 }= o (k^{-0.5 + \epsilon})$ for any $\epsilon > 0$.
\end{enumerate}
\end{theorem}

\begin{theorem} [Divergence of FO-BLO] \label{th:counter}
Let $p, r, v \in \mathbb{N}, \alpha > 0, \theta_0 \in \mathbb{R}^p$, $\{ \gamma_k > 0 \}_{k = 1}^\infty$ be any sequence satisfying (\ref{eq:step}), and $D$ be any positive number. Then there exists a set $\Omega_\mathcal{T}$ with a distribution $p(\mathcal{T})$ on it and $\mathcal{L}^{in}, \mathcal{L}^{out} : \mathbb{R}^p \times \Omega_\mathcal{D} \to \mathbb{R}$ satisfying Assumption \ref{as:liphes}, such that for $U: \mathbb{R}^p \times \mathcal{T} \to \mathbb{R}^p$ defined according to (\ref{eq:maml}), $\mathcal{M} : \mathbb{R}^p \to \mathbb{R}$ defined according to (\ref{eq:opt}) and satisfying Assumption \ref{as:mreg}, the following holds: define $\mathcal{G}_{FO}: \mathbb{R}^p \times \Omega_\mathcal{T} \to \mathbb{R}^p$ as $\mathcal{G}_{FO} (\theta, \mathcal{T}) = \nabla_{\phi_r} \mathcal{L}^{out} (\phi_r, \mathcal{T})$ where $\phi_r = U(\theta, \mathcal{T})$. Let $\{ \mathcal{T}_{k,w} \}, w \in \{ 1, \dots, v \}, k \in \mathbb{N}$ be a set of i.i.d. samples from $p(\mathcal{T})$. Let $\{ \theta_k \in \mathbb{R}^p \}_{k = 0}^\infty$ be a sequence where for all $k \in \mathbb{N}$, $\theta_k = \theta_{k - 1} - \frac{\gamma_k}{v} \sum_{w = 1}^v \mathcal{G}_{FO} (\theta_{k - 1}, \mathcal{T}_{k,w})$. Then $\liminf_{k \to \infty} \E {\| \nabla_{\theta_k} \mathcal{M} (\theta_k) \|_2^2} > D$.
\end{theorem}
\section{Comparison to other methods for BLO} \label{sec:comp}

Quantitative and qualitative comparisons between algorithms for BLO are shown in Table \ref{tab:comp}. We compare convergence guarantees for the mini-batch GD in the nonlinear, nonconvex setting.

The checkpointing technique \cite{ckpt} allows reduction of memory consumption by a $\sqrt{r}$ factor at the cost of doubling the running time, although the asymptotic time complexity is unchanged. Suggested in the context of meta-learning, iMAML \cite{imaml} modifies the definition of $U (\theta, \mathcal{T})$ in (\ref{eq:opt}) as
\begin{equation} \label{eq:imaml}
    U_{\textrm{iMAML}} (\theta, \mathcal{T}) = \argmin_{\phi \in \mathbb{R}^p} \mathcal{L}^{in} (\phi, \mathcal{T}) + (\lambda / 2) \cdot \| \phi - \theta \|_2^2 ,
\end{equation}
where $\lambda \geq 0$ is a hyperparameter. Once the optimum (\ref{eq:imaml}) is found, the gradient $\nabla_\theta U_{\textrm{iMAML}} (\theta, \mathcal{T})$ can be computed using the Implicit Function Theorem without storing the optimization loop in memory. The Implicit Function Theorem can be only applied for the exact solution of (\ref{eq:imaml}) or can serve as an approximation when the solution of (\ref{eq:imaml}) is found using an iterative solver up to some small error tolerance $\epsilon > 0$. Rather than being fixed, the running time of iMAML inner GD depends on the optimized function and the hyperparameter $\lambda$. An upper bound on the running time can only be obtained under the restrictive assumption that the objective (\ref{eq:imaml}) is a strongly-convex function for any choice of $\mathcal{T} \in \Omega_\mathcal{T}$, i.e. when $\forall \phi: \nabla_\phi^2 \mathcal{L} (\phi, \cdot) \succ - \lambda I$ \cite{imaml}. The running time of iMAML (see Table \ref{tab:comp}) depends both on the tolerance $\epsilon$ and the condition number $\kappa$ of the strongly convex objective (\ref{eq:imaml}). Consequently, iMAML requires a careful choice of $\lambda$ (possibly through an expensive grid search) in order to satisfy the strong convexity restriction. In addition, in practical scenarios which involve a neural network inside the definition of $\mathcal{L}^{in} (\phi, \mathcal{T})$, evaluation of the time complexity can be difficult as it requires computing eigenvalues of the neural network's Hessian \cite{esthes,iteig}. As pointed out in \cite{imaml}, one could alternatively use truncated back-prop \cite{truncated} to approximate outer gradients of (\ref{eq:imaml}) (Table \ref{tab:comp}).

\begin{table}[t]
  \caption{Qualitative and quantitative comparison of algorithms for BLO gradient estimation. "Convergence" column indicates whether the outer loop (Algorithm \ref{alg:sgd}) is provably converging. $\kappa$ is an upper bound on a condition number of a matrix $\nabla_\phi^2 \mathcal{L}^{in} (\phi, \mathcal{T}) + \lambda I$ for any $\phi \in \mathbb{R}^p, \mathcal{T} \in \Omega_\mathcal{T}$ and $\epsilon$ is a selected inner loop error tolerance. $\widetilde{O}$ notation hides additional logarithmic factors. For UFO-BLO in the ``Inner-loop time'' column we report expected time complexity (Corollary \ref{cor:ufomamltm}).}
  \label{tab:comp}
  \centering
  \begin{tabular}{llll}
    \toprule
    Algorithm & Convergence & Inner-loop time & Inner-loop memory \\
    \midrule
    BLO (Alg. \ref{alg:maml}) & Yes (Theorem \ref{th:conv}, $q = 1$) & $O(r p + r C_\mathrm{T})$ & $O(r p + C_\mathrm{M})$ \\
    Checkpoints \cite{ckpt} & Yes (Theorem \ref{th:conv}, $q = 1$) & $O(r p + r C_\mathrm{T})$ & $O(\sqrt{r} p + C_\mathrm{M})$ \\
    FO-BLO (Alg. \ref{alg:fomaml}) & No in general (Th. \ref{th:counter}) & $O(r p + r C_\mathrm{T})$ & $O(p + C_\mathrm{M})$ \\
    iMAML \cite{imaml} & If $\forall \phi: \nabla_\phi^2 \mathcal{L}^{in} (\phi, \cdot) \succ - \lambda I$ & $\widetilde{O}( \kappa^\frac12 (p + C_\mathrm{T}) \log \frac{1}{\epsilon} )$ & $O(p + C_\mathrm{M})$ \\
    Truncation \cite{truncated} & If $\forall \phi: \nabla_\phi^2 \mathcal{L}^{in} (\phi, \cdot) \succ - \lambda I$ & $\widetilde{O} (\kappa (p + C_\mathrm{T}) \log \frac{1}{\epsilon})$ & $\widetilde{O}(p \kappa \log \frac{1}{\epsilon} + C_\mathrm{M})$ \\
    UFO-BLO (ours) & \textbf{Yes (Theorem \ref{th:conv})} & $\boldsymbol{O(r p + r C_\mathrm{T})}$ & $\boldsymbol{O(p + C_\mathrm{M})}$ \\
    \bottomrule
  \end{tabular}
\end{table}

\section{Experiments} \label{sec:exp}

\subsection{Synthetic experiment -- simulation of FO-BLO divergence}

Theorem \ref{th:counter} is proven by explicitly constructing the following counterexample: $\Omega_\mathcal{T} = \{ \mathcal{T}_1, \mathcal{T}_2 \}$ where both tasks $\mathcal{T}_1, \mathcal{T}_2$ are equiprobable under $p(\mathcal{T})$ and functions $\mathcal{L}^{in} (\phi, \mathcal{T}_i) = \mathcal{L}^{out} (\phi, \mathcal{T}_i)$ are piecewise-polynomials of $\phi$ for $i \in \{ 1, 2 \}$ (see Appendix \ref{sec:proofs} for details). Figure \ref{fig:synth} is a simulation of this example for a case with a single parameter ($p = 1$) and inner-GD length of $r = 10$. More details and additional parameters of the simulation can be found in Appendix \ref{sec:synth}.

\subsection{Few-shot classification}

We compare UFO-BLO with other algorithms on Omniglot \cite{omniglot} and Mini-ImageNet \cite{imagenet}, popular few-shot classification benchmarks. We use MAML formulation of the few-shot meta-learning problem (Section \ref{sec:fewshot}). Both datasets consist of many classes with a few images for each class. We take train and test splits as in \cite{maml,reptile}. To sample from $p(\mathcal{T})$ in the $K$-shot $m$-way setting, $m$ classes are chosen randomly and $K + 1$ examples are drawn from each class randomly: $K$ examples for training and $1$ for testing, i.e. $s = mK, t = m$. We reuse convolutional architectures for $g(\phi, X)$ from \cite{maml} and set UFO-BLO inner-loop length to $r = 10$, as used by \cite{reptile}. In addition to exact BLO and FO-BLO, we compare with Reptile \cite{reptile} -- a modification of MAML which, similarly to FO-BLO, does not require storing inner-loop states in memory. Table \ref{tab:acc}  presents experimental results. On Omniglot, exact BLO shows the best performance on a range of setups, but is memory-inefficient (See Table \ref{tab:comp}). Out of all memory-efficient approaches (FO-BLO, Reptile \cite{reptile}, UFO-BLO), UFO-BLO with $q = 0.2$ shows the best performance in all Omniglot setups. Similarly, on Mini-ImageNet, UFO-BLO with $q = 0.2$ outperforms FO-BLO but performs slightly worse than the memory-inefficient exact BLO. More experimental details and extensions can be found in Appendix \ref{sec:expdet}.

\begin{table}[t]
  \caption{1-shot accuracy ($\%$) on Omniglot and Mini-ImageNet. The first three rows of results for the 20-way column are taken from
  \cite{reptile}.}
  \label{tab:acc}
  \centering
  \begin{tabular}{llllll}
    \toprule
    \, & \multicolumn{4}{c}{Omniglot} & Mini-ImageNet \\
    \midrule
    Algorithm & 20-way & 30-way & 40-way & 50-way & 10-way \\
    \midrule
    Exact BLO & $95.8 \pm 0.3$ & $90.8 \pm 0.2$ & $89.0 \pm 0.7$ & $87.9 \pm 0.4$ & $29.5 \pm 0.4$ \\
    \midrule
    Reptile \cite{reptile} & $89.4 \pm 0.1$ & $85.4 \pm 0.3$ & $82.5 \pm 0.3$ & $79.5 \pm 0.3$ & $31.7 \pm 0.2$ \\
    FO-BLO & $89.4 \pm 0.5$ & $81.1 \pm 1.2$ & $71.6 \pm 1.5$ & $64.4 \pm 2.1$ & $27.1 \pm 0.1$ \\
    UFO ($q = 0.1$) & $88.0 \pm 0.7$ & $84.6 \pm 0.7$ & $83.8 \pm 0.6$ & $81.6 \pm 0.7$ & $27.1 \pm 0.1$ \\
    UFO ($q = 0.2$) & $\mathbf{92.2 \pm 0.4}$ & $\mathbf{88.7 \pm 0.9}$ & $\mathbf{88.3 \pm 0.1}$ & $\mathbf{87.5 \pm 0.6}$ & $28.8 \pm 0.5$ \\
    \bottomrule
  \end{tabular}
\end{table}
\section{Related work}
\label{sec:related}
\textbf{Memory-efficient computation graphs.} Limited and expensive memory is often a bottleneck in modern massive-scale deep learning applications requiring hundreds of GPUs or TPUs employed in the training process simultaneously \cite{xlnet,megatron}. A variety of cross-domain techniques have been adopted to circumvent this issue. For instance, checkpointing \cite{gckpt,ckpt} is a generic solution to memory reduction at the cost of longer running time. A number of deep learning applications benefit from reversible architecture design allowing memory-efficient back-propagation. Among them are hyperparameter optimization \cite{revopt}, image classification \cite{revnet} with residual neural networks, autoregressive \cite{reformer} and flow-based generative modelling \cite{neuralode,nf,graphnf}. Another popular heuristic to save memory during back-propagation, though not always theoretically justified, is truncated back-propagation which is employed in bilevel optimization \cite{truncated}, recurrent neural network (RNN) \cite{tbptt}, Transformer \cite{transformerxl,xlnet} training and generalized meta-learning \cite{gilm}.

\textbf{Unbiased gradient estimation.} Stochastic gradient descent (SGD) \cite{bottou} is an essential component of large-scale machine learning. Unbiased gradient estimation, as a part of SGD, guarantees convergence to a stationary point of the optimization objective. For this reason, many algorithms were proposed to perform unbiased gradient estimation in various applications, e.g. REINFORCE \cite{reinforce} and its low-variance modifications \cite{rebar,muprop} with applications in reinforcement learning and evolution strategies \cite{nes}. The variational autoencoder \cite{autoencoder} and variational dropout \cite{vardrop} are based on a reparametrization trick for unbiased back-propagation through continuous or, involving a relaxation \cite{gumbel,concrete}, discrete random variables. Similar to this work, \cite{unbtruncated} propose an unbiased version of truncated back-propagation through the RNN. The crucial difference is that \cite{unbtruncated} propose a ``local" correction for each temporal position of the RNN with a stochastic memory reduction, while we propose to correct for the whole outer-loop iteration and manage to obtain a deterministic memory bound which is a better match for the scenario of a \textit{fixed, limited memory budget}.

\textbf{Theory of meta-learning.} Our proof technique, while supported on meta-learning benchmarks such as Omniglot and Mini-ImageNet, also fits into the realm of theoretical understanding for meta-learning, which has been explored in \cite{convergence_theory, gcsmaml} for nonconvex functions, as well as \cite{online_metalearning, provable_guarantees} for convex functions and their extensions, such as online convex optimization \cite{oco_book}. While \cite{convergence_theory} provides a brief counterexample for which $(r = 1)$-step FO-BLO does not converge, we establish a rigorous non-convergence counterexample proof for FO-BLO with any number of steps $r$ when using \textit{stochastic} gradient descent. Our proof is based on arguments using expectations and probabilities, providing new insights into stochastic optimization during meta-learning. Furthermore, while \cite{gcsmaml} touches on the \textit{zero-order} case found in \cite{esmaml}, which is mainly focused on reinforcement learning, our work studies the case where exact gradients are available, which is suited for supervised learning. 

\section{Conclusion}
\label{sec:conclusion}

We proposed unbiased first-order bilevel optimization (UFO-BLO) -- a modification of first-order bilevel optimization (FO-BLO) which incorporates unbiased gradient estimation at negligible cost (same memory and expected time complexity). UFO-BLO with a SGD-based outer loop is guaranteed to converge to a stationary point of the BLO problem while having a strictly better $O(1)$ memory complexity than the naive BLO approach. We demonstrate a rich family of BLO problems where FO-BLO ends up arbitrarily far away from the stationary point.

\section{Acknowledgements}

Adrian Weller acknowledges support from the David MacKay Newton research fellowship at Darwin College, The Alan Turing Institute under EPSRC grant EP/N510129/1 and U/B/000074, and the Leverhulme Trust via CFI.
\section{Broader Impact}
This research has a direct impact on the theoretical understanding of bilevel optimization employed in various deep-learning applications such as hyperparameter optimization, neural architecture search, adversarial robustness and gradient-based meta-learning methods (MAML), which are used in robotics, language, and vision. We have rigorously demonstrated some of the failure cases for convergence in the bilevel optimization framework, which may benefit both practitioners and theoreticians alike. Our proof techniques may also be extended for future work on the theory of gradient based adaptation. Furthermore, we have shown that our relatively simple modification is memory efficient, which can be scaled for applications, potentially allowing better democratization (i.e. cost reduction) of deep learning and reducing intensive computation usage, energy consumption \cite{energy} and $\text{CO}_2$ emission \cite{co2}.

\bibliographystyle{plain}
\bibliography{references}

\newpage
\appendix
\centerline{{\larger[2] Appendices for the paper}}
\centerline{{\larger[2] UFO-BLO: Unbiased First-Order Bilevel Optimization}}

\section{Synthetic experiment -- setup details} \label{sec:synth}

We set the following values to parameters from Theorem \ref{th:conv}, Theorem \ref{th:counter} and its proof for simulation:
\begin{gather*}
    p = 1, \quad r = 10, \quad v = 1, \quad \alpha = 0.1, \quad q = 0.1 \text{ (UFO-BLO)}, \quad \forall k \in \mathbb{N}: \gamma_k = \frac{10}{k}, \\
    a_1 = 0.5, \quad a_2 = 1.5, \quad D = 0.06 .
\end{gather*}
We do 10 simulations for both FO-BLO and UFO-BLO where we sample $\theta_0$ from a uniform distribution on a segment $[-10, 30]$.

\section{Additional experimental details and extensions} \label{sec:expdet}

We report additional results for 1-shot 15-way and 1-shot 20-way setups on Mini-ImageNet in Table \ref{tab:appendix_acc}. Standard deviations are reported for 3 runs with different seeds. All results are reported in a transductive setting \cite{reptile}. In all setups for Reptile we reuse the code from \cite{reptile}. For exact BLO and UFO-BLO we use gradient clipping so that each entry of the gradient is in $[-0.1, 0.1]$. Depending on the dataset, we use the following hyperparameters:
\begin{itemize}
    \item \textbf{Omniglot}. $\tau = 200000$ outer iterations, $\forall k: \gamma_k = 0.1$, $v = 5$, $\alpha = 0.005$. In all setups for Reptile we set hyperparameter values as in 1-shot 20-way experiment in the implementation of \cite{reptile}. If Reptile hyperparameters are set to the values used for exact BLO/FO-BLO/UFO-BLO, it shows worse performance.
    \item \textbf{Mini-ImageNet}. For all methods (FO-BLO, exact BLO, UFO-BLO, Reptile) we set: $\tau = 100000$ outer iterations, $\forall k: \gamma_k = 1$, $v = 5$, $\alpha = 0.001$ (as in transductive 1-shot 5-way Mini-ImageNet experiment of \cite{reptile}).
\end{itemize}

\begin{table}[ht]
  \caption{1-shot accuracy ($\%$) on additional runs for Mini-ImageNet. We observe a consistent pattern that UFO-BLO with $q = 0.2$ is outperforming FO-BLO and even beats Reptile algorithm on the hardest 1-shot 20-way setup.}
  \label{tab:appendix_acc}
  \centering
  \begin{tabular}{lll}
    \toprule
    Algorithm & 15-way & 20-way \\
    \midrule
    Exact BLO & $23.1 \pm 0.2$ & $19.4 \pm 0.1$ \\
    \midrule
    Reptile \cite{reptile} & $23.3 \pm 0.2$ & $15.4 \pm 2.7$ \\
    FO-BLO & $21.0 \pm 0.3$ & $17.1 \pm 0.3$ \\
    UFO ($q = 0.1$) & $20.9 \pm 1.0$ & $17.9 \pm 0.1$ \\
    UFO ($q = 0.2$) & $22.6 \pm 0.3$ & $\boldsymbol{18.7 \pm 0.2}$ \\
    \bottomrule
  \end{tabular}
\end{table}

\section{Time and memory complexity of a feed-forward neural network} \label{sec:ffn}

Consider a feed-forward neural network $g (\phi, X)$ with $R$ layers parameterized by $\phi = \{ W_1 \in \mathbb{R}^{\beta_1 \times \beta_0}, \dots, W_R \in \mathbb{R}^{\beta_R \times \beta_{R - 1}} \} \in \mathbb{R}^p$ where $p = \sum_{\gamma = 1}^R \beta_{\gamma - 1} \beta_\gamma$, $\beta_0 = n, \beta_R = m$. Let $\sigma(\cdot)$ be an elementwise nonlinearity (e.g. tanh or ReLU). Then $g(\phi, X)$ is computed as
\begin{equation} \label{eq:fnn}
    g(\phi, X) = \sigma (W_R \sigma (\dots \sigma(W_1 X) \dots))
\end{equation}
where for simplicity we consider a bias-free neural network (analogous analysis can be performed for a neural network with biases). According to (\ref{eq:fnn}), for any $\phi, X$ $Z = g(\phi, X)$ can be computed in $O(\sum_{\gamma = 1}^R \beta_{\gamma - 1} \beta_\gamma) = O(p)$ time and memory. $l_\mathrm{CCE} (Z, Y)$ can be computed in $O(\beta_R) = O(p)$ time and memory. Therefore, $\mathcal{L}_\mathrm{CCE} (\phi, \mathcal{D})$ can be computed in either $O(s \cdot p)$ or $O(t \cdot p)$ time depending on whether $\mathcal{D}$ is a train or test dataset with a universal bound of $O(\max (s, t) \cdot p)$. The upper bound on the memory requirement for $\mathcal{L}_\mathrm{CCE} (\phi, \mathcal{D})$ is $O(p)$ since each computation of $l_\mathrm{CCE} (g(\phi, X), Y)$ for $(X, Y) \in \mathcal{D}$ can use the same memory space without additional allocation.

\section{Proofs}
\label{sec:proofs}

In this Appendix we provide proofs for Theorems \ref{th:conv} and \ref{th:counter} from the main body of the paper.

\subsection{Theorem \ref{th:conv}}

We start by formulating and proving three helpful lemmas.
\begin{lemma}
\label{lemma:mlipsch}
Let $p, r, v \in \mathbb{N}, \alpha > 0, q \in (0, 1], \theta_0 \in \mathbb{R}^p$, $p(\mathcal{T})$ be a distribution on a nonempty set $\Omega_\mathcal{T}$, $\mathcal{L}^{in}, \mathcal{L}^{out} : \mathbb{R}^p \times \Omega_\mathcal{T} \to \mathbb{R}$ be functions satisfying Assumption \ref{as:liphes}, and let $U: \mathbb{R}^p \times \Omega_\mathcal{T} \to \mathbb{R}^p$ be defined according to (\ref{eq:maml}), $\mathcal{M} : \mathbb{R}^p \to \mathbb{R}$ be defined according to (\ref{eq:opt}) and satisfy Assumption \ref{as:mreg}. Then for all $\theta', \theta'' \in \mathbb{R}^p$ it holds that
\begin{equation*}
    \| \nabla_{\theta'} \mathcal{M} (\theta') - \nabla_{\theta''} \mathcal{M} (\theta'') \|_2 \leq \biggl( L_2 (1 + \alpha L_2)^{2r} + \frac{L_1 L_3}{L_2} ((1 + \alpha L_2)^{2r} - 1) \biggr) \| \theta' - \theta'' \|_2 .
\end{equation*}
\end{lemma}
\begin{proof}
Fix $\mathcal{T} = \in \Omega_\mathcal{T}$. Let $\phi_0' = \theta', \dots, \phi_r'$ and $\phi_0'' = \theta'', \dots, \phi_r''$ be inner-GD rollouts (\ref{eq:maml}) for $\theta'$ and $\theta''$ respectively. For each $1 \leq j \leq r$  inequalities applies:
\begin{align*}
    \| \phi_j' - \phi_j'' \|_2 &= \| \phi_{j - 1}' - \phi_{j - 1}'' - \alpha (\nabla_{\phi_{j - 1}'} \mathcal{L}^{in} (\phi_{j - 1}', \mathcal{T}) - \nabla_{\phi_{j - 1}''} \mathcal{L}^{in} (\phi_{j - 1}'', \mathcal{T})) \|_2 \\
    &\leq \| \phi_{j - 1}' - \phi_{j - 1}'' \|_2 + \alpha \| \nabla_{\phi_{j - 1}'} \mathcal{L}^{in} (\phi_{j - 1}', \mathcal{T}) - \nabla_{\phi_{j - 1}''} \mathcal{L}^{in} (\phi_{j - 1}'', \mathcal{T}) \|_2 \\
    &\leq \| \phi_{j - 1}' - \phi_{j - 1}'' \|_2 + \alpha L_2 \| \phi_{j - 1}' - \phi_{j - 1}'' \|_2 \\
    &= (1 + \alpha L_2) \| \phi_{j - 1}' - \phi_{j - 1}'' \|_2.
\end{align*}
where we use Lipschitz-continuity of $\nabla_\phi \mathcal{L}^{in} (\phi, \mathcal{T})$ with respect to $\phi$ with Lipschitz constant $L_2$ (upper bound on $\nabla_\phi^2 \mathcal{L} (\phi, \mathcal{D})$). Therefore, for each $0 \leq j \leq r$
\begin{equation*}
    \| \phi_j' - \phi_j'' \|_2 \leq (1 + \alpha L_2)^j \| \phi_0' - \phi_0'' \|_2 = (1 + \alpha L_2)^j \| \theta' - \theta'' \|_2
\end{equation*}
and
\begin{equation*}
    \| \nabla_{\phi_r'} \mathcal{L}^{out} (\phi_r', \mathcal{T})) - \nabla_{\phi_r''} \mathcal{L}^{out} (\phi_r'', \mathcal{T}) \|_2 \leq L_2 \| \phi_r' - \phi_r'' \|_2 \leq L_2 (1 + \alpha L_2)^r \| \theta' - \theta'' \|_2.
\end{equation*}

For each $1 \leq j \leq r$ the following chain of inequalities applies as a result of (\ref{eq:mamlgrad}):
\begin{align}
    &\| \nabla_{\phi_{j - 1}'} \mathcal{L}^{out} (\phi_r', \mathcal{T}) - \nabla_{\phi_{j - 1}''} \mathcal{L}^{out} (\phi_r'', \mathcal{T}) \|_2 = \| \nabla_{\phi_j'} \mathcal{L}^{out} (\phi_r', \mathcal{T}) - \nabla_{\phi_j''} \mathcal{L}^{out} (\phi_r'', \mathcal{T}) \nonumber \\
    &- \alpha ( \nabla_{\phi_{j - 1}'}^2 \mathcal{L}^{in} (\phi_{j - 1}', \mathcal{T}) \nabla_{\phi_j'} \mathcal{L}^{out} (\phi_r', \mathcal{T}) - \nabla_{\phi_{j - 1}''}^2 \mathcal{L}^{in} (\phi_{j - 1}'', \mathcal{T}) \nabla_{\phi_j''} \mathcal{L}^{out} (\phi_r'', \mathcal{T}) ) \|_2 \nonumber \\
    &= \| \nabla_{\phi_j'} \mathcal{L}^{out} (\phi_r', \mathcal{T}) - \nabla_{\phi_j''} \mathcal{L}^{out} (\phi_r'', \mathcal{T}) - \alpha \nabla_{\phi_{j - 1}'}^2 \mathcal{L}^{in} (\phi_{j - 1}', \mathcal{T}) ( \nabla_{\phi_j'} \mathcal{L}^{out} (\phi_r', \mathcal{T}) \nonumber \\
    &- \nabla_{\phi_j''} \mathcal{L}^{out} (\phi_r'', \mathcal{T}) ) - \alpha (\nabla_{\phi_{j - 1}'}^2 \mathcal{L}^{in} (\phi_{j - 1}', \mathcal{T}) - \nabla_{\phi_{j - 1}''}^2 \mathcal{L}^{in} (\phi_{j - 1}'', \mathcal{T}) ) \nabla_{\phi_j''} \mathcal{L}^{out} (\phi_r'', \mathcal{T}) \|_2 \nonumber \\
    &\leq \| \nabla_{\phi_j'} \mathcal{L}^{out} (\phi_r', \mathcal{T}) - \nabla_{\phi_j''} \mathcal{L}^{out} (\phi_r'', \mathcal{T}) \|_2 + \alpha \| \nabla_{\phi_{j - 1}'}^2 \mathcal{L}^{in} (\phi_{j - 1}', \mathcal{T}) \|_2 \| \nabla_{\phi_j'} \mathcal{L}^{out} (\phi_r', \mathcal{T}) \nonumber \\
    &- \nabla_{\phi_j''} \mathcal{L}^{out} (\phi_r'', \mathcal{T}) \|_2 + \alpha \| \nabla_{\phi_{j - 1}'}^2 \mathcal{L}^{in} (\phi_{j - 1}', \mathcal{T}) - \nabla_{\phi_{j - 1}''}^2 \mathcal{L}^{in} (\phi_{j - 1}'', \mathcal{T}) \|_2 \nonumber \cdot \| \nabla_{\phi_j''} \mathcal{L}^{out} (\phi_r'', \mathcal{T}) \|_2 \nonumber \\
    &\leq (1 + \alpha L_2) \| \nabla_{\phi_j'} \mathcal{L}^{out} (\phi_r', \mathcal{T}) - \nabla_{\phi_j''} \mathcal{L}^{out} (\phi_r'', \mathcal{T}) \|_2 + \alpha L_1 L_3 \| \phi_{j - 1}' - \phi_{j - 1}'' \|_2 \nonumber \\
    &\leq (1 + \alpha L_2) \| \nabla_{\phi_j'} \mathcal{L}^{out} (\phi_r', \mathcal{T}) - \nabla_{\phi_j''} \mathcal{L}^{out} (\phi_r'', \mathcal{T}) \|_2 + \alpha L_1 L_3 (1 + \alpha L_2)^{j - 1} \| \theta' - \theta'' \|_2 . \label{eq:gradineq}
\end{align}
By unfolding inequality (\ref{eq:gradineq}) for $1 \leq j \leq r$ we obtain that
\begin{align*}
    &\| \nabla_{\theta'} \mathcal{L}^{out} (U(\theta', \mathcal{T}), \mathcal{T}) - \nabla_{\theta''} \mathcal{L}^{out} (U(\theta'', \mathcal{T}), \mathcal{T}) \|_2 = \| \nabla_{\phi_0'} \mathcal{L}^{out} (\phi_r', \mathcal{T}) \\
    &- \nabla_{\phi_0''} \mathcal{L}^{out} (\phi_r'', \mathcal{T}) \|_2 \leq \biggl( L_2 (1 + \alpha L_2)^{2r} + \alpha L_1 L_3 \sum_{j = 0}^{r - 1} (1 + \alpha L_2)^{2j} \biggr) \| \theta' - \theta'' \|_2 \\
    &= \biggl( L_2 (1 + \alpha L_2)^{2r} + \frac{L_1 L_3}{L_2} ((1 + \alpha L_2)^{2r} - 1) \biggr) \| \theta' - \theta'' \|_2 .
\end{align*}
Finally, by taking expectation with respect to $\mathcal{T} \sim p(\mathcal{T})$ and applying Jensen inequality we get
\begin{align*}
    \| \nabla_{\theta'} \mathcal{M} (\theta') - \nabla_{\theta''} \mathcal{M} (\theta'') \|_2^2 &\leq \mathbb{E}_{p(\mathcal{T})} \left[\| \nabla_{\theta'} \mathcal{L}^{out} (U(\theta', \mathcal{T}), \mathcal{T}) - \nabla_{\theta''} \mathcal{L}^{out} (U(\theta'', \mathcal{T}), \mathcal{T}) \|_2^2  \right] \\
    &\leq \biggl( L_2 (1 + \alpha L_2)^{2r} + \frac{L_1 L_3}{L_2} ((1 + \alpha L_2)^{2r} - 1) \biggr)^2 \| \theta' - \theta'' \|_2^2
\end{align*}
which is equivalent to the statement of Lemma.
\end{proof}
\begin{lemma} \label{lemma:unbbnd}
Let $p, r, v \in \mathbb{N}, \alpha > 0, q \in (0, 1], \theta_0 \in \mathbb{R}^p$, $p(\mathcal{T})$ be a distribution on a nonempty set $\Omega_\mathcal{T}$, $\mathcal{L}^{in}, \mathcal{L}^{out} : \mathbb{R}^p \times \Omega_\mathcal{T} \to \mathbb{R}$ be functions satisfying Assumption \ref{as:liphes}, and let $U: \mathbb{R}^p \times \Omega_\mathcal{T} \to \mathbb{R}^p$ be defined according to (\ref{eq:maml}), $\mathcal{M} : \mathbb{R}^p \to \mathbb{R}$ be defined according to (\ref{eq:opt}) and satisfy Assumption \ref{as:mreg}. Define $\mathcal{G}: \mathbb{R}^p \times \Omega_\mathcal{T} \times \{ 0, 1 \} \to \mathbb{R}^p$ as
\begin{equation*}
    \mathcal{G} (\theta, \mathcal{T}, x) = \nabla_{\phi_r} \mathcal{L}^{out} (\phi_r, \mathcal{T}) + (x / q) (\nabla_\theta \mathcal{L}^{out} (\phi_r, \mathcal{T}) - \nabla_{\phi_r} \mathcal{L}^{out} (\phi_r, \mathcal{T}))
\end{equation*}
where $\phi_r = U(\theta, \mathcal{T})$. Then for all $\theta \in \mathbb{R}^p$
\begin{gather}
    \mathbb{E}_{\xi, p(\mathcal{T})} \left[\mathcal{G} (\theta, \mathcal{T}, \xi) \right]= \nabla_\theta \mathcal{M} (\theta), \label{eq:unbgrad} \\
    \mathbb{E}_{\xi, p(\mathcal{T})} \left[\| \mathcal{G} (\theta, \mathcal{T}, \xi) \|_2^2 \right] \leq \biggl( 1 + q^{-1} ((1 + \alpha L_2)^r - 1) \biggr)^2 L_1^2. \label{eq:bndvar}
\end{gather}
\end{lemma}
\begin{proof}
(\ref{eq:unbgrad}) is satisfied by observing that
\begin{align*}
    \mathbb{E}_{\xi, p(\mathcal{T})} \left[\mathcal{G} (\theta, \mathcal{T}, \xi) \right] &= \mathbb{E}_{p(\mathcal{T})} \biggl[ \mathbb{E}_\xi \left[\mathcal{G}(\theta, \mathcal{T}, \xi) \right] \biggr] \\
    &= \mathbb{E}_{p(\mathcal{T})} \biggl[ \nabla_{\phi_r} \mathcal{L}^{out} (\phi_r, \mathcal{T}) + \frac{q}{q} (\nabla_\theta \mathcal{L}^{out} (\phi_r, \mathcal{T}) - \nabla_{\phi_r} \mathcal{L}^{out} (\phi_r, \mathcal{T})) \biggr] \\
    &= \mathbb{E}_{p(\mathcal{T})} \biggl[ \nabla_\theta \mathcal{L}^{out} (\phi_r, \mathcal{T}) \biggr] = \mathbb{E}_{p(\mathcal{T})} \biggl[ \nabla_\theta \mathcal{L}^{out} (U(\theta, \mathcal{T}), \mathcal{T}) \biggr] \\
    &= \nabla_\theta \mathbb{E}_{p(\mathcal{T})} \biggl[ \mathcal{L}^{out} (U(\theta, \mathcal{T}), \mathcal{T}) \biggr] = \nabla_\theta \mathcal{M} (\theta).
\end{align*}

To show (\ref{eq:bndvar}), we fix $\mathcal{T} \in \Omega_\mathcal{T}$ and observe that by Assumption \ref{as:liphes}
\begin{equation*}
    \| \nabla_{\phi_r} \mathcal{L}^{out} (\phi_r, \mathcal{T}) \|_2 \leq L_1
\end{equation*}
and according to (\ref{eq:mamlgrad}) for each $1 \leq j \leq r$
\begin{align*}
    \| \nabla_{\phi_{j - 1}} \mathcal{L}^{out} (\phi_r, \mathcal{T}) \|_2 &= \| \nabla_{\phi_j} \mathcal{L}^{out} (\phi_r, \mathcal{T}) - \alpha \nabla^2_{\phi_{j - 1}} \mathcal{L}^{in} (\phi_{j - 1}, \mathcal{T})^\top \nabla_{\phi_j} \mathcal{L}^{out} (\phi_r, \mathcal{T}) \|_2 \\
    &= \| \nabla_{\phi_j} \mathcal{L}^{out} (\phi_r, \mathcal{T}) \|_2 + \alpha \| \nabla^2_{\phi_{j - 1}} \mathcal{L}^{in} (\phi_{j - 1}, \mathcal{T}) \|_2 \| \nabla_{\phi_j} \mathcal{L}^{out} (\phi_r, \mathcal{T}) \|_2 \\
    &\leq (1 + \alpha L_2) \| \nabla_{\phi_j} \mathcal{L}^{out} (\phi_r, \mathcal{T}) \|_2.
\end{align*}
Therefore,
\begin{align}
    \| \nabla_\theta \mathcal{L}^{out} (\phi_r, \mathcal{T}) \|_2 &= \| \nabla_{\phi_0} \mathcal{L}^{out} (\phi_r, \mathcal{T}) \|_2 \leq (1 + \alpha L_2)^r \| \nabla_{\phi_r} \mathcal{L}^{out} (\phi_r, \mathcal{T}) \|_2 \nonumber \\
    &\leq (1 + \alpha L_2)^r L_1, \nonumber \\
    \| \mathcal{G} (\theta, \mathcal{T}, \xi) \|_2 &\leq \max (\| \mathcal{G} (\theta, \mathcal{T}, 0) \|_2, \| \mathcal{G} (\theta, \mathcal{T}, 1) \|_2) \nonumber \\
    &\leq \max (L_1, \| (1 - q^{-1}) \nabla_{\phi_r} \mathcal{L}^{out} (\phi_r, \mathcal{T}) + q^{-1} \nabla_\theta \mathcal{L}^{out} (\phi_r, \mathcal{T}) \|_2) \nonumber \\
    &\leq \max (L_1, (1 - q^{-1}) \| \nabla_{\phi_r} \mathcal{L}^{out} (\phi_r, \mathcal{T}) \|_2 + q^{-1} \| \nabla_\theta \mathcal{L}^{out} (\phi_r, \mathcal{T}) \|_2) \nonumber \\
    &\leq \max (L_1, (1 - q^{-1}) L_1 + q^{-1} (1 + \alpha L_2)^r L_1 ) \nonumber \\
    &= \max (1, 1 + q^{-1} ((1 + \alpha L_2)^r - 1) ) L_1 \nonumber \\
    &= \biggl( 1 + q^{-1} ((1 + \alpha L_2)^r - 1) \biggr) L_1 \label{eq:l2last}
\end{align}
where we use $q^{-1} ((1 + \alpha L_2)^r - 1) \geq 0$. (\ref{eq:bndvar}) is obtained by squaring (\ref{eq:l2last}) and taking expectation with respect to $\mathcal{T}$ and $\xi$.
\end{proof}

\begin{lemma} \label{th:convexp}
Let $p, r, v \in \mathbb{N}, \alpha > 0, q \in (0, 1], \theta_0 \in \mathbb{R}^p$, $p(\mathcal{T})$ be a distribution on a nonempty set $\Omega_\mathcal{T}$, $\{ \gamma_k > 0 \}_{k = 1}^\infty$ be any sequence, $\mathcal{L}^{in}, \mathcal{L}^{out} : \mathbb{R}^p \times \Omega_\mathcal{T} \to \mathbb{R}$ be functions satisfying Assumption \ref{as:liphes}, and let $U: \mathbb{R}^p \times \Omega_\mathcal{T} \to \mathbb{R}^p$ be defined according to (\ref{eq:maml}), $\mathcal{M} : \mathbb{R}^p \to \mathbb{R}$ be defined according to (\ref{eq:opt}) and satisfy Assumption \ref{as:mreg}. Define $\mathcal{G}: \mathbb{R}^p \times \Omega_\mathcal{T} \times \{ 0, 1 \} \to \mathbb{R}^p$ as
\begin{equation*}
    \mathcal{G} (\theta, \mathcal{T}, x) = \nabla_{\phi_r} \mathcal{L}^{out} (\phi_r, \mathcal{T}) + (x / q) (\nabla_\theta \mathcal{L}^{out} (\phi_r, \mathcal{T}) - \nabla_{\phi_r} \mathcal{L}^{out} (\phi_r, \mathcal{T}))
\end{equation*}
where $\phi_r = U(\theta, \mathcal{T})$. Let $\{ \mathcal{T}_{k,w} \}, \{ \xi_{k,w} \}, w \in \{ 1, \dots, v \}, k \in \mathbb{N}$ be sets of i.i.d. samples from $p(\mathcal{T})$ and $\mathrm{Bernoulli} (q)$ respectively, such that $\sigma$-algebras populated by $\{ \mathcal{T}_{k,w} \}_{\forall k, w}, \{ \xi_{k,w} \}_{\forall k, w}$ are independent.
Let $\{ \theta_k \in \mathbb{R}^p \}_{k = 0}^\infty$ be a sequence where for all $k \in \mathbb{N}$ $\theta_k = \theta_{k - 1} - \frac{\gamma_k}{v} \sum_{w = 1}^v \mathcal{G} (\theta_{k - 1}, \mathcal{T}_{k,w}, \xi_{k,w})$.
Then for each $k \in \mathbb{N}$
\begin{equation} \label{eq:convres}
    \sum_{u = 1}^k \gamma_u \mathbb{E} \left[ \| \nabla_{\theta_{u - 1}} \mathcal{M} (\theta_{u - 1}) \|_2^2 \right] \leq \mathcal{M} (\theta_0) - \mathcal{M}^* + \mathcal{C} \sum_{u = 1}^k \gamma_u^2
\end{equation}
where
\begin{align*}
    \mathcal{C} &= \frac{1}{2 v} \biggl( ( 1 + q^{-1} ((1 + \alpha L_2)^r - 1) )^2 + (v - 1) (1 + \alpha L_2)^{2r} \biggr) L_1^2 \\
    &\times \biggl( L_2 (1 + \alpha L_2)^{2r} + \frac{L_1 L_3}{L_2} ((1 + \alpha L_2)^{2r} - 1) \biggr).
\end{align*}
\end{lemma}

\begin{proof}
Denote
\begin{align*}
    A &= \biggl( L_2 (1 + \alpha L_2)^{2r} + \frac{L_1 L_3}{L_2} ((1 + \alpha L_2)^{2r} - 1) \biggr) \| \theta' - \theta'' \|_2, \\
    B &= \biggl( 1 + q^{-1} ((1 + \alpha L_2)^r - 1) \biggr)^2 L_1^2, \\
    C &= (1 + \alpha L_2)^{2 r} L_1^2.
\end{align*}
For each $\theta \in \mathbb{R}^p$ we apply Jensen's inequality to obtain
\begin{equation*}
    \| \nabla \mathcal{M} (\theta) \|_2^2 \leq \mathbb{E}_{p (\mathcal{T})} \left[\| \nabla_\theta \mathcal{L}^{out} (U(\theta, \mathcal{T}), \mathcal{T}) \|_2^2 \right] \leq C
\end{equation*}
where the second inequality is according to (\ref{eq:bndvar}) of Lemma \ref{lemma:unbbnd} when $q$ is set to $1$. Fix $u \in \mathbb{N}$ and denote $\overline{\mathcal{G}} = \frac{1}{v} \sum_{w = 1}^v \mathcal{G} (\theta_{u - 1}, \mathcal{T}_{u,w}, \xi_{u,w})$. Let $\mathcal{F}_u$ be a $\sigma$-algebra populated by $\{ \mathcal{T}_{\kappa,w} \}, \{ \xi_{\kappa,w} \}, w \in \{ 1, \dots, v \}, \kappa < u$. From Lemma \ref{lemma:unbbnd} it follows that
\begin{align*}
    &\E { \overline{\mathcal{G}} | \mathcal{F}_u } = \nabla \mathcal{M} (\theta_{u - 1}), \\
    &\E { \| \overline{\mathcal{G}} \|_2^2 | \mathcal{F}_u } = \frac{1}{v^2} \sum_{w = 1}^v \sum_{w' = 1}^v \E{ \mathcal{G} (\theta_{u - 1}, \mathcal{T}_{u,w}, \xi_{u,w})^\top \mathcal{G} (\theta_{u - 1}, \mathcal{T}_{u,w'}, \xi_{u,w'}) | \mathcal{F}_u } \\
    &= \frac{1}{v^2} (v \E { \| \mathcal{G} (\theta_{u - 1}, \mathcal{T}_{u,1}, \xi_{u,1}) \|_2^2 | \mathcal{F}_u } + v (v - 1) \| \E { \mathcal{G} (\theta_{u - 1}, \mathcal{T}_{u,1}, \xi_{u,1}) | \mathcal{F}_u } \|_2^2) \\
    &= \frac{1}{v} (\E { \| \mathcal{G} (\theta_{u - 1}, \mathcal{T}_{u,1}, \xi_{u,1}) \|_2^2 | \mathcal{F}_u } + (v - 1) \| \mathcal{M} (\theta_{u - 1}) \|_2^2) \\
    &\leq \frac{1}{v} (B + (v - 1) C)
\end{align*}

$A$ is a Lipschitz constant for $\mathcal{M}$'s gradients (Lemma \ref{lemma:mlipsch}). We apply Inequality 4.3 from \cite{bottou} to obtain that for all $\theta', \theta'' \in \mathbb{R}^p$
\begin{equation*}
    \mathcal{M} (\theta') \leq \mathcal{M} (\theta'') + \nabla \mathcal{M} (\theta'')^\top (\theta' - \theta'') + \frac{1}{2} A \| \theta' - \theta'' \|_2^2.
\end{equation*}
By setting $\theta' = \theta_u$, $\theta'' = \theta_{u - 1}$ we deduce that
\begin{equation*}
    \mathcal{M} (\theta_u) \leq \mathcal{M} (\theta_{u - 1}) - \gamma_u \nabla \mathcal{M} (\theta_{u - 1})^\top \overline{\mathcal{G}} + \frac{1}{2} \gamma_u^2 A \| \overline{\mathcal{G}} \|_2^2.
\end{equation*}
Then
\begin{align}
    \E { \mathcal{M} (\theta_u) | \mathcal{F}_u } &\leq \mathcal{M} (\theta_{u - 1}) - \gamma_u \nabla \mathcal{M} (\theta_{u - 1})^\top \E { \overline{\mathcal{G}} | \mathcal{F}_u } + \frac{1}{2} \gamma_u^2 A \E { \| \overline{\mathcal{G}} \|_2^2 | \mathcal{F}_u } \nonumber \\
    &\leq \mathcal{M} (\theta_{u - 1}) - \gamma_u \| \nabla \mathcal{M} (\theta_{u - 1}) \|_2^2 + \frac{1}{2 v} \gamma_u^2 (B + (v - 1) C). \label{eq:sgdnonlin}
\end{align}
Take full expectation of (\ref{eq:sgdnonlin}) and observe that $\mathcal{C} = \frac{1}{2 v} (B + (v - 1) C)$:
\begin{equation*}
    \E{\mathcal{M} (\theta_u)} \leq \E{\mathcal{M} (\theta_{u - 1})} - \gamma_u \E{\| \nabla \mathcal{M} (\theta_{u - 1}) \|_2^2} + \gamma_u^2 \mathcal{C}
\end{equation*}
which is equivalent to
\begin{equation} \label{eq:sgdnonlinexp}
    \gamma_u \E{\| \nabla \mathcal{M} (\theta_{u - 1}) \|_2^2} \leq \E{ \mathcal{M} (\theta_{u - 1})} - \E{\mathcal{M} (\theta_u)} + \gamma_u^2 \mathcal{C}.
\end{equation}
Sum inequalities (\ref{eq:sgdnonlinexp}) for all $1 \leq u \leq k$:
\begin{align*}
    \sum_{u = 1}^k \gamma_u \E{\| \nabla \mathcal{M} (\theta_{u - 1}) \|_2^2} &\leq \mathcal{M} (\theta_0) - \E{\mathcal{M} (\theta_k)} + \mathcal{C} \sum_{u = 1}^k\gamma_u^2 \\
    &\leq \mathcal{M} (\theta_0) - \mathcal{M}^* + \mathcal{C} \sum_{u = 1}^k\gamma_u^2,
\end{align*}
and the proof is concluded.
\end{proof}

\begin{proof}[Theorem \ref{th:conv} proof]

Under conditions of the theorem results of Lemma \ref{th:convexp} are true.

First, we prove 1. If $\sum_{k = 1}^\infty \gamma_k^2 < \infty$, then the right-hand side of (\ref{eq:convres}) converges to a finite value when $k \to \infty$. Therefore, the left-hand side also converges to a finite value. Suppose the statement of 1 is false. Then there exists $k_0 \in \mathbb{N}, A > 0$ such that $\forall u \geq k_0 : \| \nabla_{\theta_{u - 1}} \mathcal{M} (\theta_{u - 1}) \|_2^2 > A$. But then for all $k \geq k_0$
\begin{equation*}
    \sum_{u = 1}^k \gamma_u \| \nabla_{\theta_{u - 1}} \mathcal{M} (\theta_{u - 1}) \|_2^2 \geq A \sum_{u = k_0}^k \gamma_u \to \infty
\end{equation*}
when $k \to \infty$, which is a contradiction. Therefore, 1 is true.

Next, we prove 2. Observe that
\begin{equation*}
    \min_{0 \leq u < k} \E{\| \nabla_{\theta_u} \mathcal{M} (\theta_u) \|_2^2} \sum_{u = 1}^k \gamma_u \leq \sum_{u = 1}^k \gamma_u \E{\| \nabla_{\theta_{u - 1}} \mathcal{M} (\theta_{u - 1}) \|_2^2} \leq \mathcal{M} (\theta_0) - \mathcal{M}^* + \mathcal{C} \sum_{u = 1}^k \gamma_u^2.
\end{equation*}
Divide by $\sum_{u = 1}^k \gamma_u$:
\begin{equation*}
    \min_{0 \leq u < k} \E{ \| \nabla_{\theta_u} \mathcal{M} (\theta_u) \|_2^2} \leq \frac{1}{\sum_{u = 1}^k \gamma_u} (\mathcal{M}(\theta_0) - \mathcal{M}^*) + \mathcal{C} \frac{1}{\sum_{u = 1}^k \gamma_u} \cdot \sum_{u = 1}^k \gamma_u^2 .
\end{equation*}
2 is satisfied by observing that
\begin{equation*}
    \sum_{u = 1}^k \gamma_u = \sum_{u = 1}^k u^{-0.5} = \Omega (k^{0.5}), \quad \sum_{u = 1}^k \gamma_u^2 = \sum_{u = 1}^k u^{-1} = O (\log k) = o (k^\epsilon)
\end{equation*}
for any $\epsilon > 0$.
\end{proof}

\subsection{Theorem \ref{th:counter}}

\begin{proof}
Consider a set $\Omega_\mathcal{T}$ consisting of two elements: $\Omega_\mathcal{T} = \{ \mathcal{T}_1, \mathcal{T}_2 \}$. Define $p(\mathcal{T})$ so that
\begin{equation*}
    \mathbb{P}_{p(\mathcal{T})}(\mathcal{T} = \mathcal{T}_1) = \mathbb{P}_{p(\mathcal{T})}(\mathcal{T} = \mathcal{T}_2) = \frac{1}{2}.
\end{equation*}
Choose arbitrary numbers $0 < a_1, a_2 < \frac{1}{\alpha}$, $a_1 \neq a_2$ and set $b_1 = 0$. Since $a_1 \neq a_2$, $(1 - \alpha a_1)/(1 - \alpha a_2) \neq 1$ and, consequently,
\begin{equation*}
    \biggl( \frac{1 - \alpha a_1}{1 - \alpha a_2} \biggr)^r \neq \biggl( \frac{1 - \alpha a_1}{1 - \alpha a_2} \biggr)^{2 r}.
\end{equation*}
Multiply by $\frac{a_1}{a_2} \neq 0$:
\begin{equation} \label{eq:ineq1}
    \frac{a_1}{a_2} \cdot \biggl( \frac{1 - \alpha a_1}{1 - \alpha a_2} \biggr)^r \neq \frac{a_1}{a_2} \cdot \biggl( \frac{1 - \alpha a_1}{1 - \alpha a_2} \biggr)^{2 r}.
\end{equation}
From (\ref{eq:ineq1}) and since $\frac{a_1}{a_2} (\frac{1 - \alpha a_1}{1 - \alpha a_2})^r, \frac{a_1}{a_2} (\frac{1 - \alpha a_1}{1 - \alpha a_2})^{2 r} > 0$ it follows that
\begin{equation*}
    \frac{\frac{a_1}{a_2} (\frac{1 - \alpha a_1}{1 - \alpha a_2})^{2 r} + 1}{\frac{a_1}{a_2} (\frac{1 - \alpha a_1}{1 - \alpha a_2})^r + 1} - 1 \neq 0.
\end{equation*}
Multiply inequality by $(1 - \alpha a_2)^{2 r} \neq 0$ and numerator/denominator by $a_2 (1 - \alpha a_2)^r \neq 0$:
\begin{equation*}
    \frac{a_1 (1 - \alpha a_1)^{2 r} + a_2 (1 - \alpha a_2)^{2 r}}{a_1 (1 - \alpha a_1)^r + a_2 (1 - \alpha a_2)^r} (1 - \alpha a_2)^r - (1 - \alpha a_2)^{2 r} \neq 0.
\end{equation*}
Because of the inequality above, we can define a number $b_2$ as
\begin{equation} \label{eq:b2def}
    b_2 = 2 \sqrt{2 D} \biggl| \frac{a_1 (1 - \alpha a_1)^{2 r} + a_2 (1 - \alpha a_2)^{2 r}}{a_1 (1 - \alpha a_1)^r + a_2 (1 - \alpha a_2)^r} (1 - \alpha a_1)^r - (1 - \alpha a_2)^{2 r} \biggr|^{-1} > 0
\end{equation}
and select arbitrary number $A$ so that
\begin{equation} \label{eq:adef}
    A > | \frac{b_1}{a_1} - \frac{b_2}{a_2} | .
\end{equation}
Consider two functions $f_i (x)$, $f_i: \mathbb{R} \to \mathbb{R}$, $i \in \{ 1, 2 \}$ defined as follows (denote $z_i = z_i (x) = | x - \frac{b_i}{a_i} |$)
\begin{equation} \label{eq:fdef}
    f_i (x) = 
    \begin{cases}
        \frac{1}{2} a_i z_i^2 & \text{if } z_i \leq A \\
        - \frac{1}{6} a_i (z_i - A)^3 + \frac{1}{2} a_i (z_i - A)^2 + a_i A z_i - \frac{1}{2} a_i A^2 & \text{if } A < z_i \leq A + 1 \\
        (\frac{1}{2} a_i + a_i A) z_i - \frac{1}{6} a_i - \frac{1}{2} a_i A^2 - \frac{1}{2} a_i A & \text{if } A + 1 < z_i
    \end{cases}.
\end{equation}
It is easy to check that for $i \in \{ 1, 2 \}$ $f_i (x)$ is twice differentiable with a global minimum at $\frac{b_i}{a_i}$. The following expressions apply for the first and second derivative:
\begin{align}
    f_i' (x) &= 
    \begin{cases}
        a_i x - b_i & \text{if } z_i \leq A \\
        \biggl( - \frac{1}{2} a_i (z_i - A)^2 + a_i z_i \biggr) \textrm{sign} (x - \frac{b_i}{a_i}) & \text{if } A < z_i \leq A + 1 \\
        (\frac{1}{2} a_i + a_i A) \textrm{sign} (x - \frac{b_i}{a_i}) & \text{if } A + 1 < z_i
    \end{cases}, \label{eq:fdif1} \\
    f_i'' (x) &= 
    \begin{cases}
        a_i & \text{if } z_i \leq A \\
        - a_i z_i + a_i + a_i A & \text{if } A < z_i \leq A + 1 \\
        0 & \text{if } A + 1 < z_i
    \end{cases}. \label{eq:fdif2}
\end{align}
From (\ref{eq:fdif1}-\ref{eq:fdif2}) it follows that each $f_i$ has bounded, Lipschitz-continuous gradients and Hessians. Define $\mathcal{L}^{in} (\phi, \mathcal{T}_i) = \mathcal{L}^{out} (\phi, \mathcal{T}_i) = f_i (\phi^{(1)})$ for $i \in \{ 1, 2 \}$, where $\phi^{(1)}$ denotes a first element of $\phi$, then Assumption \ref{as:liphes} is satisfied. Since $\Omega_\mathcal{T}$ is finite, Assumption \ref{as:mreg} is also satisfied.

Let $I = [\frac{b_2}{a_2} - A, \frac{b_1}{a_1} + A]$. Observe that from (\ref{eq:adef}) it follows that $\frac{b_1}{a_1}, \frac{b_2}{a_2} \in I$ and $I \subseteq [\frac{b_i}{a_i} - A, \frac{b_i}{a_i} + A]$ for $i \in \{ 1, 2 \}$, i.e. $I$ corresponds to a quadratic part of both $f_1 (x)$ and $f_2 (x)$. If $x \in I$, then for $i \in \{ 1, 2 \}$
\begin{align}
    x - \alpha f'_i (x) &= x - \alpha (a_i x - b_i) = (1 - \alpha a_i) x + \alpha b_i \nonumber \\
    &= (1 - \alpha a_i) \cdot x + \alpha a_i \cdot \frac{b_i}{a_i} \in [\min(x, \frac{b_i}{a_i}), \max(x, \frac{b_i}{a_i})] \subseteq I \label{eq:iprop}
\end{align}
since $x - \alpha f'_i (x)$ is a convex combination of $x$ and $\frac{b_i}{a_i}$ ($0 < \alpha a_i, 1 - \alpha a_i < 1$). From (\ref{eq:iprop}) and the definition of $\mathcal{L}^{in} (\phi, \mathcal{T}), \mathcal{L}^{out} (\phi, \mathcal{T})$ it follows that if $\phi_0, \dots, \phi_r$ is a rollout of inner GD (\ref{eq:maml}) for task $\mathcal{T}_i$ and $\phi_0^{(1)} \in I$, then $\phi_1^{(1)}, \dots, \phi_r^{(1)} \in I$ and, hence,
\begin{gather}
    \nabla_{\phi_r} \mathcal{L}^{out} (\phi_r, \mathcal{T}_i)^{(1)} = f_i' (\phi_r^{(1)}) = a_i \phi_r^{(1)} - b_i, \nonumber \\
    \forall j \in \{ 1, \dots, r \}: \phi_j^{(1)} = (1 - \alpha a_i) \phi_{j - 1}^{(1)} + \alpha b_i. \label{eq:mamlcnt}
\end{gather}
From (\ref{eq:mamlcnt}) we derive that
\begin{gather}
    \phi_j^{(1)} - \frac{b_i}{a_i} = (1 - \alpha a_i) (\phi_{j - 1}^{(1)} - \frac{b_i}{a_i}), \quad \phi_r^{(1)} - \frac{b_i}{a_i} = (1 - \alpha a_i)^r (\phi_0^{(1)} - \frac{b_i}{a_i}), \nonumber \\
    \phi_r^{(1)} = (1 - \alpha a_i)^r (\phi_0^{(1)} - \frac{b_i}{a_i}) + \frac{b_i}{a_i}, \nonumber \\
    \nabla_{\phi_r} \mathcal{L}^{out} (\phi_r, \mathcal{T}_i)^{(1)} = a_i \biggl( (1 - \alpha a_i)^r (\phi_0^{(1)} - \frac{b_i}{a_i}) + \frac{b_i}{a_i} \biggr) - b_i = a_i (1 - \alpha a_i)^r ( \phi_0^{(1)} - \frac{b_i}{a_i} ) . \label{eq:fomamlcnt}
\end{gather}

From (\ref{eq:step}) it follows that there exists a deterministic number $k_0 \in \mathbb{N}$ such that for all $k \geq k_0$
\begin{equation} \label{eq:k0def1}
    \gamma_k < \frac{1}{2} \min_{i \in \{ 1, 2 \} } \frac{1}{a_i (1 + \alpha a_i)^r} .
\end{equation}
If (\ref{eq:k0def1}) holds, then it also holds that
\begin{equation} \label{eq:k0def2}
    \gamma_k < \min_{i \in \{ 1, 2 \} } \frac{1}{a_i (1 + \alpha a_i)^r}, \quad \gamma_k < \min_{i \in \{ 1, 2 \} } \frac{1}{a_i (1 - \alpha a_i)^r}.
\end{equation}
For any $k \geq k_0$ the following cases are possible:
\begin{enumerate}
    \item Case 1: $\theta_{k - 1}^{(1)} \in I$. An identity (\ref{eq:fomamlcnt}) allows to write that for $i \in \{ 1, 2 \}$
\begin{equation} \label{eq:fomamloutcnt}
    \mathcal{G}_{FO} (\theta_{k - 1}, \mathcal{T}_i)^{(1)} = a_i (1 - \alpha a_i)^r ( \theta_{k - 1}^{(1)} - \frac{b_i}{a_i} ).
\end{equation}
For $i \in \{ 1, 2 \}$ let random number $v_i \leq v$ denote a number of tasks in $\mathcal{T}_{k,1}, \dots, \mathcal{T}_{k,v}$ which coincide with $\mathcal{T}_i$. Then from (\ref{eq:fomamloutcnt}) we deduce that
\begin{align*}
    \theta_k^{(1)} &= \theta_{k - 1}^{(1)} - \gamma_k \sum_{i = 1}^2 \frac{v_i}{v} a_i (1 - \alpha a_i)^r ( \theta_{k - 1}^{(1)} - \frac{b_i}{a_i} ) \\
    &= (1 - \gamma_k \sum_{i = 1}^2 \frac{v_i}{v} a_i (1 - \alpha a_i)^r ) \cdot \theta_{k - 1}^{(1)} + \gamma_k \frac{v_1}{v} a_1 (1 - \alpha a_1)^r \cdot \frac{b_1}{a_1} \\
    &+ \gamma_k \frac{v_2}{v} a_2 (1 - \alpha a_2)^r \cdot \frac{b_2}{a_2} \\
    &\in [\min(\theta_{k - 1}^{(1)}, \frac{b_1}{a_1}, \frac{b_2}{a_2}), \max(\theta_{k - 1}^{(1)}, \frac{b_1}{a_1}, \frac{b_2}{a_2})] \subseteq I
\end{align*}
since $\theta_k^{(1)}$ is a convex combination of $\theta_{k - 1}^{(1)}, \frac{b_1}{a_1}, \frac{b_2}{a_2}$. Indeed, due to (\ref{eq:k0def2})
\begin{equation*}
    0 \leq (1 - \gamma_k \sum_{i = 1}^2 \frac{v_i}{v} a_i (1 - \alpha a_i)^r ), \gamma_k \frac{v_1}{v} a_1 (1 - \alpha a_1)^r, \gamma_k \frac{v_2}{v} a_2 (1 - \alpha a_2)^r \leq 1
\end{equation*}
and
\begin{equation*}
    (1 - \gamma_k \sum_{i = 1}^2 \frac{v_i}{v} a_i (1 - \alpha a_i)^r ) + \gamma_k \frac{v_1}{v} a_1 (1 - \alpha a_1)^r + \gamma_k \frac{v_2}{v} a_2 (1 - \alpha a_2)^r = 1.
\end{equation*}
As a result of this Case we conclude that if $k \geq k_0$ and $\theta_{k - 1}^{(1)} \in I$, then for all $k' \geq k$ it also holds that $\theta_{k'}^{(1)} \in I$.
\item Case 2: $\theta_{k - 1}^{(1)} > \frac{b_1}{a_1} + A$. From (\ref{eq:fdif2}) observe that for $i \in \{ 1, 2 \}$ and any $x \in \mathbb{R}$ $f''_i (x) \leq a_i$. Hence, $f'_i$'s Lipschitz constant is $a_i$. Let $\phi_0, \dots, \phi_r$ and $\overline{\phi}_0, \dots, \overline{\phi}_r$ be two inner-GD (\ref{eq:maml}) rollouts for task $\mathcal{T}_i$ and $\phi_0^{(1)} > \overline{\phi}_0^{(1)}$. For $j \in \{ 1, \dots, r \}$ suppose that $\phi_{j - 1}^{(1)} > \overline{\phi}_{j - 1}^{(1)}$. Then
\begin{align*}
    \phi_j^{(1)} - \overline{\phi}_j^{(1)} &= \phi_{j - 1}^{(1)} - \overline{\phi}_{j - 1}^{(1)} - \alpha (f_i' (\phi_{j - 1}^{(1)}) - f_i'(\overline{\phi}_{j - 1}^{(1)})) \\
    &\geq \phi_{j - 1}^{(1)} - \overline{\phi}_{j - 1}^{(1)} - \alpha | f_i' (\phi_{j - 1}^{(1)}) - f_i'(\overline{\phi}_{j - 1}^{(1)}) | \\
    &\geq \phi_{j - 1}^{(1)} - \overline{\phi}_{j - 1}^{(1)} - \alpha a_i | \phi_{j - 1}^{(1)} - \overline{\phi}_{j - 1}^{(1)} | \\
    &> \phi_{j - 1}^{(1)} - \overline{\phi}_{j - 1}^{(1)} - | \phi_{j - 1}^{(1)} - \overline{\phi}_{j - 1}^{(1)} | \\
    &= 0
\end{align*}
or $\phi_j^{(1)} > \overline{\phi}_j^{(1)}$ where we use Lipschitz continuity of $f'_i$ and that $\alpha a_i < 1$ by the choice of $a_1, a_2$. Therefore, since $\phi_0^{(1)} > \overline{\phi}_0^{(1)}$, $\phi_1^{(1)} > \overline{\phi}_1^{(1)}$ and so on, eventually $\phi_r^{(1)} > \overline{\phi}_r^{(1)}$. Observe that $f_i' (x)$ is a strictly monotonously increasing function, therefore $f_i' (\phi_r^{(1)}) > f_i' (\overline{\phi}_r^{(1)})$. To sum up:
\begin{equation} \label{eq:mon}
    f_i' (\phi_r^{(1)}) > f_i' (\overline{\phi}_r^{(1)}) \quad \text{when } \phi_0^{(1)} > \overline{\phi}_0^{(1)}.
\end{equation}

Set $\overline{\phi}_0^{(1)} = \frac{b_i}{a_i}$, then $f'(\overline{\phi}_{j - 1}^{(1)}) = 0$ and $\overline{\phi}_1^{(1)} = \overline{\phi}_0^{(1)} - \alpha \cdot 0 = \overline{\phi}_0^{(1)}$ and so on, eventually $\overline{\phi}_r^{(1)} = \frac{b_i}{a_i}$ and $f'(\overline{\phi}_r^{(1)}) = 0$. Therefore, if $\phi_0^{(1)} = \frac{b_1}{a_1} + A > \max ( \frac{b_1}{a_1}, \frac{b_2}{a_2} )$ then $f_i' (\phi_r^{(1)}) > f_i' (\overline{\phi}_r^{(1)}) = 0$. For $i \in \{ 1, 2 \}$ denote a deterministic value of $f_i' (\phi_r^{(1)})$ by $B_i  > 0$. By setting $\phi_0^{(1)} = \theta_{k - 1}^{(1)}, \overline{\phi}_0^{(1)} = \frac{b_1}{a_1} + A$ and using (\ref{eq:mon}) we obtain:
\begin{equation} \label{eq:glb}
    \mathcal{G}_{FO} (\theta_{k - 1}, \mathcal{T}_i)^{(1)} = f'_i (\phi_r^{(1)}) > f'_i (\overline{\phi}_r^{(1)}) = B_i \geq B > 0 .
\end{equation}
where we denote $B = \min (B_1, B_2)$. 

In addition, set $\phi_0^{(1)} = \theta_{k - 1}^{(1)}, \overline{\phi}_0^{(1)} = \frac{b_i}{a_i}$. Then
\begin{align*}
    \mathcal{G}_{FO} (\theta_{k - 1}, \mathcal{T}_i)^{(1)} &= | \mathcal{G}_{FO} (\theta_{k - 1}, \mathcal{T}_i)^{(1)} | = | f'_i (\phi_r^{(1)}) - 0 | = | f'_i (\phi_r^{(1)}) - f'_i (\overline{\phi}_r^{(1)}) | \\
    &\leq a_i | \phi_r^{(1)} - \overline{\phi}_r^{(1)} | = a_i | \phi_{r - 1}^{(1)} - \overline{\phi}_{r - 1}^{(1)} - \alpha ( f'_i (\phi_{r - 1}^{(1)}) - f'_i (\overline{\phi}_{r - 1}^{(1)}) ) | \\
    &\leq a_i | \phi_{r - 1}^{(1)} - \overline{\phi}_{r - 1}^{(1)} | + \alpha a_i | f'_i (\phi_{r - 1}^{(1)}) - f'_i (\overline{\phi}_{r - 1}^{(1)}) | \\
    &\leq a_i (1 + \alpha a_i) | \phi_{r - 1}^{(1)} - \overline{\phi}_{r - 1}^{(1)} | \\
    &\dots \\
    &\leq a_i (1 + \alpha a_i)^r | \phi_0^{(1)} - \overline{\phi}_0^{(1)} | \\
    &= a_i (1 + \alpha a_i)^r | \theta_{k - 1}^{(1)} - \frac{b_i}{a_i} |.
\end{align*}
Since $\theta_{k - 1}^{(1)} > \frac{b_1}{a_1} + A > \max (\frac{b_1}{a_1}, \frac{b_2}{a_2})$, we derive that
\begin{align*}
    \mathcal{G}_{FO} (\theta_{k - 1}, \mathcal{T}_i)^{(1)} &\leq a_i (1 + \alpha a_i)^r (\theta_{k - 1}^{(1)} - \frac{b_i}{a_i}) \leq \frac{1}{\gamma_k} (\theta_{k - 1}^{(1)} - \frac{b_i}{a_i}) \\
    &\leq \max_{i' \in \{ 1, 2 \}} \frac{1}{\gamma_k} (\theta_{k - 1}^{(1)} - \frac{b_{i'}}{a_{i'}}) = \frac{1}{\gamma_k} (\theta_{k - 1}^{(1)} - \min_{i' \in \{ 1, 2 \}} \frac{b_{i'}}{a_{i'}}) \\
    &= \frac{1}{\gamma_k} (\theta_{k - 1}^{(1)} - \frac{b_1}{a_1})
\end{align*}
where we use (\ref{eq:k0def2}) and the fact that $\frac{b_1}{a_1} = 0, \frac{b_2}{a_2} > 0$. Next, we deduce that
\begin{equation*}
    \frac{1}{v} \sum_{w = 1}^v \mathcal{G}_{FO} (\theta_{k - 1}, \mathcal{T}_{k,w})^{(1)} \leq \frac{1}{v} \sum_{w = 1}^v \frac{1}{\gamma_k} (\theta_{k - 1}^{(1)} - \frac{b_1}{a_1}) = \frac{1}{\gamma_k} (\theta_{k - 1}^{(1)} - \frac{b_1}{a_1})
\end{equation*}
and
\begin{equation} \label{eq:b1}
    \theta_k^{(1)} = \theta_{k - 1}^{(1)} - \frac{\gamma_k}{v} \sum_{w = 1}^v \mathcal{G}_{FO} (\theta_{k - 1}, \mathcal{T}_{k,w})^{(1)} \geq \theta_{k - 1}^{(1)} - \frac{\gamma_k}{\gamma_k} (\theta_{k - 1}^{(1)} - \frac{b_1}{a_1}) = \frac{b_1}{a_1}.
\end{equation}
On the other hand, from (\ref{eq:glb}) we observe that
\begin{equation*}
    \frac{1}{v} \sum_{w = 1}^v \mathcal{G}_{FO} (\theta_{k - 1}, \mathcal{T}_{k,w})^{(1)} > \frac{1}{v} \sum_{w = 1}^v B = B
\end{equation*}
and
\begin{equation} \label{eq:b2}
    \theta_k^{(1)} = \theta_{k - 1}^{(1)} - \frac{\gamma_k}{v} \sum_{w = 1}^v \mathcal{G}_{FO} (\theta_{k - 1}, \mathcal{T}_{k,w})^{(1)} < \theta_{k - 1}^{(1)} - \gamma_k B .
\end{equation}
According to (\ref{eq:step}) there exists a number $k_1 > k$ such that
\begin{equation} \label{eq:k1def}
    \sum_{k' = k}^{k_1 - 1} \gamma_{k'} > \frac{1}{B} (\theta^{(1)}_{k - 1} - \frac{b_1}{a_1}).
\end{equation}
In addition, let $k_1$ be a minimal such number. Suppose that for all $k \leq k' \leq k_1$ $\theta_{k' - 1}^{(1)} > \frac{b_1}{a_1} + A$. Then by applying bound (\ref{eq:b2}) for all $k = k'$ we obtain that
\begin{equation*}
    \theta_{k_1}^{(1)} < \theta_{k_1 - 1}^{(1)} - \gamma_{k_1} B < \dots < \theta_{k - 1}^{(1)} - B \sum_{k' = k}^{k_1} \gamma_{k'} < \frac{b_1}{a_1}
\end{equation*}
which is a contradiction with the bound (\ref{eq:b1}) applied to $k = k_1$. Therefore, there exists $k \leq k' < k_1$ such that $\theta_{k'}^{(1)} \leq \frac{b_1}{a_1} + A$. Then there exists a number
\begin{equation} \label{eq:k2def}
    k_2 = \min_{k \leq k' < k_1, \theta_{k'}^{(1)} \leq \frac{b_1}{a_1} + A} k' .
\end{equation}
Hence, $\theta_{k_2 - 1}^{(1)} > \frac{b_1}{a_1} + A$ and by applying bound (\ref{eq:b1}) to $k = k_2$ we conclude that $\theta_{k_2}^{(1)} \geq \frac{b_1}{a_1}$. Averall:
\begin{equation*}
    \theta_{k_2}^{(1)} \in [\frac{b_1}{a_1}, \frac{b_1}{a_1} + A] \subseteq I.
\end{equation*}
As shown in Case 1, for all $k' > k_2$ (including $k_1$) it also holds that $\theta_{k'}^{(1)} \in I$. To summarize, we have proven that there exists a deterministic number $B > 0$ such that for $k_1$ defined by (\ref{eq:k1def}) $\theta_{k'}^{(1)} \in I$ for all $k' \geq k_1$.

\item Case 3: $\theta_{k - 1}^{(1)} < \frac{b_2}{a_2} - A$. Using a symmetric argument as in Case 2 it can be shown that there exists a deterministic number $C > 0$ so that the following holds. According to (\ref{eq:step}) there exists $k_3 \geq k$ such that
\begin{equation} \label{eq:k3def}
    \sum_{k' = k}^{k_3 - 1} \gamma_{k'} > \frac{1}{C} (\frac{b_2}{a_2} - \theta^{(1)}_{k - 1}).
\end{equation}
In addition, let $k_3$ be a minimal such number. Then $\theta_{k'}^{(1)} \in I$ for all $k' \geq k_3$.
\end{enumerate}

Since $p(\mathcal{T})$ is a discrete distribution, there only exists a finite number of outcomes for a set of random variables $\{ \mathcal{T}_{k,w} \}_{k < k_0, 1 \leq w \leq v }$. Therefore, there is only a finite set of possible outcomes of $\theta_{k_0 - 1}^{(1)}$ random variable. Consequently, there exists a deterministic number $E > 0$ such that $| \theta_{k_0 - 1}^{(1)} | < E$. According to (\ref{eq:step}) there exist deterministic numbers $k_4, k_5 \geq k_0$ such that
\begin{equation} \label{eq:k45def}
    \sum_{k' = k_0}^{k_4 - 1} \gamma_{k'} > \frac{1}{B} (E - \frac{b_1}{a_1}), \quad \sum_{k' = k_0}^{k_5 - 1} \gamma_{k'} > \frac{1}{C} (\frac{b_2}{a_2} + E).
\end{equation}
and let $k_6 = \max (k_4, k_5)$ -- also a deterministic number. Let $k_1, k_3$ be random numbers from Cases 2, 3 applied to $k = k_0$. Then from (\ref{eq:k45def}) and $E$'s definition it follows that $k_1, k_3 \leq k_6$. In addition to that, $k_0 \leq k_6$ from $k_6$'s definition. As a result of all Cases we conclude that for any $k \geq k_6$ $\phi_k^{(1)} \in I$.

Denote
\begin{gather*}
    a^* = \frac{1}{2} ( a_1 (1 - \alpha a_1)^r + a_2 (1 - \alpha a_2)^r ), \quad b^* = \frac{1}{2} (b_1 (1 - \alpha a_1)^r + b_2 (1 - \alpha a_2)^r ),  \quad x^* = \frac{b^*}{a^*}
\end{gather*}
and consider arbitrary $k > k_6$. Denote $\overline{\mathcal{G}} = \frac{1}{v} \sum_{w = 1}^v \mathcal{G}_{FO} (\theta_{k - 1}, \mathcal{T}_{k,w})$ and let $\mathcal{F}_k$ be a $\sigma$-algebra populated by $\{ \mathcal{T}_{\kappa,w} \}, w \in \{ 1, \dots, v \}, \kappa < k$. From Equation (\ref{eq:fomamloutcnt}) we conclude that
\begin{equation*}
    \E{ \overline{\mathcal{G}}^{(1)} | \mathcal{F}_k } = \frac{1}{v} \sum_{w = 1}^v \E {\mathcal{G}_{FO} (\theta_{k - 1}, \mathcal{T}_{k,w})^{(1)} | \mathcal{F}_k } = a^* \theta_{k - 1}^{(1)} - b^*
\end{equation*}
Outer-loop update leads to an expression:
\begin{equation*}
    \theta_k^{(1)} = \theta_{k - 1}^{(1)} - \gamma_k \overline{\mathcal{G}}^{(1)}.
\end{equation*}
Subtract $x^*$:
\begin{equation*}
    \theta_k^{(1)} - x^* = \theta_{k - 1}^{(1)} - x^* - \gamma_k \overline{\mathcal{G}}^{(1)}.
\end{equation*}
Take a square:
\begin{align*}
    ( \theta_k^{(1)} - x^* )^2 &= ( \theta_{k - 1}^{(1)} - x^* - \gamma_k \overline{\mathcal{G}}^{(1)} )^2
    &= (\theta_{k - 1}^{(1)} - x^*)^2 - 2 \gamma_k (\theta_{k - 1}^{(1)} - x^*) \overline{\mathcal{G}}^{(1)} + \gamma_k^2 \overline{\mathcal{G}}^{(1) 2} .
\end{align*}
Take expectation conditioned on $\mathcal{F}_k$:
\begin{align*}
    \E{ ( \theta_k^{(1)} - x^* )^2 | \mathcal{F}_k} &= ( \theta_{k - 1}^{(1)} - x^* )^2 - 2 \gamma_k (\theta_{k - 1}^{(1)} - x^*) \E{ \overline{\mathcal{G}}^{(1)} | \mathcal{F}_k} + \gamma_k^2 \E {  \left(\overline{\mathcal{G}}^{(1)}\right)^{2} | \mathcal{F}_k} \\
    &= ( \theta_{k - 1}^{(1)} - x^* )^2 - 2 \gamma_k (\theta_{k - 1}^{(1)} - x^*) (a^* \theta_{k - 1}^{(1)} - b^*) + \gamma_k^2 \E { \left(\overline{\mathcal{G}}^{(1)} \right)^{2} | \mathcal{F}_k} \\
    &= ( \theta_{k - 1}^{(1)} - x^* )^2 - 2 \gamma_k a^* (\theta_{k - 1}^{(1)} - x^*) ( \theta_{k - 1}^{(1)} - \frac{b^*}{a^*}) + \gamma_k^2 \E { \left(\overline{\mathcal{G}}^{(1)}\right)^{2} | \mathcal{F}_k} \\
    &= (1 - 2 \gamma_k a^*) ( \theta_{k - 1}^{(1)} - x^* )^2 + \gamma_k^2 \E{ \left(\overline{\mathcal{G}}^{(1)}\right)^{2} | \mathcal{F}_k}.
\end{align*}
For $i \in \{ 1, 2 \}$, $\mathcal{G}_{FO} (\theta_{k - 1}, \mathcal{T}_i^{(1)})$ depends linearly on $\theta_{k - 1}^{(1)}$ (\ref{eq:fomamloutcnt}) and, therefore, is bounded on $\theta_{k - 1}^{(1)} \in I$. Hence, $\overline{\mathcal{G}}^{(1) 2}$ is also bounded by a deterministic number $F > 0$: $\overline{\mathcal{G}}^{(1) 2} < F$. Then:
\begin{equation*}
    \E{ ( \theta_k^{(1)} - x^* )^2 | \mathcal{F}_k} \leq (1 - 2 \gamma_k a^*) ( \theta_{k - 1}^{(1)} - x^* )^2 + \gamma_k^2 F.
\end{equation*}
Take a full expectation:
\begin{equation*}
    \E{( \theta_k^{(1)} - x^* )^2} \leq (1 - 2 \gamma_k a^*) \mathbb{E} ( \theta_{k - 1}^{(1)} - x^* )^2 + \gamma_k^2 F,
\end{equation*}
and denote $y_k = \E{ ( \theta_k^{(1)} - x^* )^2}$:
\begin{equation} \label{eq:yineq}
    y_k \leq (1 - 2 \gamma_k a^*) y_{k - 1} + \gamma_k^2 F.
\end{equation}
Now, we prove that $\lim_{k \to \infty} y_k = 0$. Indeed, consider arbitrary $\epsilon > 0$. According to (\ref{eq:step}) there exists $k_\epsilon > k_6$ such that $\forall k' \geq k_\epsilon : \gamma_{k'} \leq \frac{a^* \epsilon}{F}$. As a result of (\ref{eq:yineq}) for every $k \geq k_\epsilon$ it holds
\begin{equation*}
    y_k \leq (1 - 2 \gamma_k a^*) y_{k - 1} + \gamma_k^2 F \leq (1 - 2 \gamma_k a^*) y_{k - 1} + \gamma_k a^* \epsilon.
\end{equation*}
Subtract $\frac{\epsilon}{2}$:
\begin{equation} \label{eq:yineq2}
    y_k - \frac{\epsilon}{2} \leq (1 - 2 \gamma_k a^*) ( y_{k - 1} - \frac{\epsilon}{2} ) .
\end{equation}
Observe that by (\ref{eq:k0def1}), $a^*$'s definition and since $k \geq k_0$ it holds that $1 - 2 \gamma_k a^* > 0$. Therefore and since (\ref{eq:yineq2}) holds for all $k \geq k_\epsilon$, it can be written that for all $k \geq k_\epsilon$
\begin{equation*}
    y_k - \frac{\epsilon}{2} \leq \biggl( \prod_{k' = k_\epsilon}^k (1 - 2 \gamma_{k'} a^*) \biggr) ( y_{k_\epsilon - 1} - \frac{\epsilon}{2} ) \leq \biggl( \prod_{k' = k_\epsilon}^k (1 - 2 \gamma_{k'} a^*) \biggr) | y_{k_\epsilon - 1} - \frac{\epsilon}{2} | .
\end{equation*}
We use inequality $1 - x \leq \exp(- x)$ to deduce that
\begin{align}
    y_k - \frac{\epsilon}{2} &\leq \biggl( \prod_{k' = k_\epsilon}^k (1 - 2 \gamma_{k'} a^*) \biggr) | y_{k_\epsilon - 1} - \frac{\epsilon}{2} | \nonumber \\
    &\leq \exp \biggl( - 2 a^* \sum_{k' = k_\epsilon}^k \gamma_{k'} \biggr) | y_{k_\epsilon - 1} - \frac{\epsilon}{2} | . \label{eq:yineq3}
\end{align}
If $ | y_{k_\epsilon - 1} - \frac{\epsilon}{2} | = 0$, then from (\ref{eq:yineq3}) it follows that $y_k \leq 0 + \frac{\epsilon}{2} < \epsilon$ for all $k \geq k_\epsilon$. Otherwise, from (\ref{eq:step}) there exists $k_\epsilon'$ such that $\sum_{k' = k_\epsilon}^{k_\epsilon'} \gamma_{k'} > \frac{\log | y_{k_\epsilon - 1} - \frac{\epsilon}{2} | - \log \frac{\epsilon}{2} }{2 a^*}$. Then from (\ref{eq:yineq3}) it follows that for all $k \geq k_\epsilon'$ $y_k - \frac{\epsilon}{2} < \frac{\epsilon}{2}$ or $y_k < \epsilon$. Since $y_k \geq 0$ by definition, we have proven that $\lim_{k \to \infty} y_k = 0$, or
\begin{equation} \label{eq:lim}
    \lim_{k \to \infty} \E {(\theta_k^{(1)} - x^* )^2} = 0 .
\end{equation}
Again, let $k > k_6$. Let $\phi_0 = \theta_k, \dots, \phi_r$ be a rollout (\ref{eq:maml}) of inner GD for task $\mathcal{T}_i$. Then according to (\ref{eq:mamlgrad})
\begin{equation*}
    \nabla_{\theta_k} \mathcal{L}^{out} (U (\theta_k, \mathcal{T}_i), \mathcal{T}_i)^{(1)} = \mathcal{G}_{FO} (\theta_k, \mathcal{T}_i)^{(1)} \prod_{j = 0}^{r - 1} (1 - \alpha f_i'' (\phi_j^{(1)})) .
\end{equation*}
From (\ref{eq:iprop}) it follows that $f_i'' (\phi_j^{(1)}) = a_i$ for $j \in \{ 0, \dots, r - 1 \}$. Moreover, we use (\ref{eq:fomamloutcnt}) to obtain that
\begin{equation*}
    \nabla_{\theta_k} \mathcal{L}^{out} (U (\theta_k, \mathcal{T}_i), \mathcal{T}_i)^{(1)} = a_i (1 - \alpha a_i)^{2 r} ( \theta_k^{(1)} - \frac{b_i}{a_i} )
\end{equation*}
and
\begin{align*}
    \nabla_{\theta_k} \mathcal{M} (\theta_k)^{(1)} &= \mathbb{E}_{p(\mathcal{T})} \left[\nabla_{\theta_k} \mathcal{L}^{out} (U (\theta_k, \mathcal{T}), \mathcal{T})^{(1)} \right] \\
    &= \frac{1}{2} \biggl( a_1 (1 - \alpha a_1)^{2 r} ( \theta_k^{(1)} - \frac{b_1}{a_1} ) + a_2 (1 - \alpha a_2)^{2 r} ( \theta_k^{(1)} - \frac{b_2}{a_2} ) \biggr) \\
    &= \widehat{a} \theta_k^{(1)} - \widehat{b}
\end{align*}
where
\begin{equation*}
    \widehat{a} = \frac{1}{2} (a_1 (1 - \alpha a_1)^{2 r} + a_2 (1 - \alpha a_2)^{2 r}), \quad \widehat{b} = \frac{1}{2}(b_1 (1 - \alpha a_1)^{2 r} + b_2 (1 - \alpha a_2)^{2 r}).
\end{equation*}
Notice that since $b_1 = 0$ and $b_2$ is defined by (\ref{eq:b2def}), it appears that $| \widehat{a} x^* - \widehat{b} | = \sqrt{2 D}$, or $(\widehat{a} x^* - \widehat{b})^2 = 2 D$. Multiply (\ref{eq:lim}) by $\widehat{a}$ to obtain that
\begin{align}
    \lim_{k \to \infty} \mathbb{E} \Big[ (\widehat{a} \theta_k^{(1)} &- \widehat{a} x^* )^2 \Big] = 0, \\
    \lim_{k \to \infty} \mathbb{E} \Big[( \widehat{a} \theta_k^{(1)} - \widehat{b} &- (\widehat{a} x^* - \widehat{b}))^2 \Big] = 0, \nonumber \\
    \lim_{k \to \infty} \mathbb{E} \Big[ ( \nabla_{\theta_k} \mathcal{M} (\theta_k)^{(1)} &- (\widehat{a} x^* - \widehat{b}))^2 \Big]  = 0 . \label{eq:lim2}
\end{align}
For each $k \geq 1$ by Jensen's inequality :
\begin{equation*}
     \left( \E {\nabla_{\theta_k} \mathcal{M} (\theta_k)^{(1)}} - (\widehat{a} x^* - \widehat{b}) \right)^2 \leq  \E{ (\nabla_{\theta_k} \mathcal{M} (\theta_k)^{(1)} - (\widehat{a} x^* - \widehat{b}))^2}.
\end{equation*}
Hence,
\begin{equation*}
    \lim_{k \to \infty} \left( \E {\nabla_{\theta_k} \mathcal{M} (\theta_k)^{(1)} } - (\widehat{a} x^* - \widehat{b})) \right)^2 = 0, \quad \lim_{k \to \infty} \E{\nabla_{\theta_k} \mathcal{M}^{(1)}} = (\widehat{a} x^* - \widehat{b})
\end{equation*}
and by expanding (\ref{eq:lim2}) we derive that
\begin{equation*}
    \lim_{k \to \infty} \E{ \left(\nabla_{\theta_k} \mathcal{M}^{(1)}\right)^{2} } = 2 (\widehat{a} x^* - \widehat{b}) \lim_{k \to \infty} \E{\nabla_{\theta_k} \mathcal{M}^{(1)}} - (\widehat{a} x^* - \widehat{b})^2 = (\widehat{a} x^* - \widehat{b})^2 = 2 D.
\end{equation*}
We conclude the proof by observing that
\begin{equation*}
    \liminf_{k \to \infty}\E{ \| \nabla_{\theta_k} \mathcal{M}\|^2_2 } =  \lim_{k \to \infty} \E{\| \nabla_{\theta_k} \mathcal{M} \|^2_2} = \lim_{k \to \infty} \E{ \left(\nabla_{\theta_k} \mathcal{M}^{(1)} \right)^{2} } = 2 D > D.
\end{equation*}
\end{proof}

\end{document}